%% file: AORPO-ijcai.tex
\documentclass{article}
\usepackage{ijcai21}

\usepackage{times}

\usepackage{soul}
\usepackage{url}
\usepackage[hidelinks]{hyperref}
\usepackage[utf8]{inputenc}
\usepackage[small]{caption}
\usepackage{graphicx}
\usepackage{amsmath}
\usepackage{booktabs}
\usepackage{dsfont}
\usepackage{balance} 
\usepackage{booktabs}       
\usepackage{amsfonts}       
\usepackage{nicefrac}       
\usepackage{microtype}      
\usepackage{amssymb}
\usepackage{appendix}
\usepackage{enumerate}
\usepackage{amsthm}
\usepackage{mathtools}
\usepackage{longtable}
\usepackage{subcaption}
\usepackage[ruled,linesnumbered,vlined]{algorithm2e}
\usepackage{makecell}
\usepackage{xcolor}
\usepackage{float}
\usepackage{wrapfig}
\usepackage{booktabs} 
\usepackage{thmtools,thm-restate}
\usepackage{enumitem}
\usepackage{cleveref} 
\usepackage{hyperref}
\hypersetup{
    colorlinks=true,
    linkcolor=blue
}

\theoremstyle{definition}

\newtheorem{theorem}{Theorem}[]
\newtheorem{extheorem}{Extended Theorem}[section]
\newtheorem{lemma}{Lemma}[section]
\newtheorem{exlemma}{Extended Lemma}[section]

\allowdisplaybreaks[3]

\input{command.tex}
\newcommand{\methodname}{AORPO} 
\newcommand{\methodnametwo}{AORDPG}
\newcommand{\fullmethodname}{Adaptive Opponent-wise Rollout Policy Optimization}

\title{Model-based Multi-agent Policy Optimization \\with Adaptive Opponent-wise Rollouts}

\author{
Weinan Zhang\and
Xihuai Wang\and
Jian Shen\and
Ming Zhou\\
\affiliations
Shanghai Jiao Tong University\\
\emails
\{wnzhang, leoxhwang, r\_ocky, mingak\}@sjtu.edu.cn
}

\begin{document}
\maketitle 


\begin{abstract}
This paper investigates the model-based methods in multi-agent reinforcement learning (MARL). We specify the dynamics sample complexity and the opponent sample complexity in MARL, and conduct a theoretic analysis of return discrepancy upper bound. To reduce the upper bound with the intention of low sample complexity during the whole learning process, we propose a novel decentralized model-based MARL method, named Adaptive Opponent-wise Rollout Policy Optimization (AORPO). In AORPO, each agent builds its multi-agent environment model, consisting of a dynamics model and multiple opponent models, and trains its policy with the adaptive opponent-wise rollout. We further prove the theoretic convergence of AORPO under reasonable assumptions. Empirical experiments on competitive and cooperative tasks demonstrate that AORPO can achieve improved sample efficiency with comparable asymptotic performance over the compared MARL methods.
\end{abstract}

\input{chapters/introduction}

\input{chapters/relatedwork}

\input{chapters/background}

\input{chapters/bound}

\input{chapters/method}

\input{chapters/experiment}

\input{chapters/conclusion}

\section*{Acknowledgements}
This work is supported by ``New Generation of AI 2030'' Major Project (2018AAA0100900) and National Natural Science Foundation of China (62076161, 61632017).
Xihuai Wang is supported by Wu Wen Jun Honorary Doctoral Scholarship, AI Institute, Shanghai Jiao Tong University.
We thank Yaodong Yang for helpful discussion.

\bibliographystyle{named}
\bibliography{reference}

\input{chapters/appendix}


\end{document}

%% file: command.tex
\newcommand{\eq}[1]{Eq.~(\ref{#1})}
\newcommand{\alg}[1]{Algorithm~\ref{#1}}
\newcommand{\fig}[1]{Figure~\ref{#1}}
\newcommand{\tab}[1]{Table~\ref{#1}}
\newcommand{\app}[1]{Appendix~\ref{#1}}
\newcommand{\se}[1]{Section~\ref{#1}}

\newcommand{\minisection}[1]{\paragraph{#1.}}

\newcommand{\cites}[1]{\citeauthor{#1}~\shortcite{#1}}

%% file: chapters/introduction.tex
\section{Introduction}\label{sec: intro}

Multi-agent reinforcement learning (MARL) has been recently paid much attention 
and preliminarily applied to various control scenarios including robot system, autonomous driving, resource utilization etc. \cite{du2020survey}.
One main technical challenge MARL brings over single-agent reinforcement learning (SARL) is that the agents need to interact with other agents, and their returns depend on the behavior of all the agents, which usually requires a considerable amount of samples, i.e., high sample complexity. 



In a general MARL setting, one agent can access the actions taken by other agents in any state through some communication protocol, without any knowledge of their specific policies \cite{tian2020learning}. In such a situation, we claim that the sample complexity in MARL comes from two parts: \textit{dynamics sample complexity} and \textit{opponent sample complexity}. To achieve some policy performance, dynamics sample complexity represents the number of interactions made with the environment dynamics while collecting the samples.
On the other hand, opponent sample complexity, which is unique in MARL, represents the times of communications to access one opponent's action. Thus one goal of MARL is to find an effective policy with the two sample complexities being low.


In SARL scenarios, it is well known that model-based reinforcement learning (MBRL) can achieve lower dynamics sample complexity than model-free reinforcement learning (MFRL) empirically \cite{wang2019benchmarking} or achieve at least competitive complexity theoretically \cite{jin2018q}. Specifically, by building a dynamics model, the agent can also be trained with the model simulation samples and the ones collected from the environment, which can reduce the need for the environment samples \cite{luo2019algorithmic}. In multi-agent scenarios, however, to utilize the dynamics model for data simulation, we need to ask for the opponents' actions through a communication protocol, which will count as the opponent sample complexity. Via building opponent models \cite{he2016opponent},  we can replace the real opponents in the data simulation stage to reduce the opponent sample complexity.




In this paper, we investigate model-based methods for MARL and propose a novel method called \fullmethodname{} (\methodname{}). Specifically, from the perspective of each ego agent\footnote{In this paper, each studied agent is called ego agent while the other agents are called opponent agents. See Figure~\ref{fig: ma-mbrl} for illustrations.}, we build a multi-agent environment model, which consists of a dynamics model and opponent models for each opponent agent. The multi-agent environment model can then be used to perform simulation rollout for MARL training to reduce both sample complexities. To our knowledge, however, in literature, there is no theoretical guidance regarding how to perform the multi-agent simulated rollouts. We provide a theoretical analysis about the upper bound of the return discrepancy w.r.t. policy distribution shift and the generalization errors of the environment dynamics and each opponent model. The upper bound reveals that when the environment dynamics model performs rollout with multiple opponents models, different magnitudes of opponent model errors may lead to different contributions to the compounding error of the multi-step simulated rollouts. Based on the theoretic analysis, we design the adaptive opponent-wise rollout scheme for our policy optimization algorithm, called \methodname{}, and prove its convergence under reasonable assumptions.

Experiments show that \methodname{} outperforms several state-of-the-art MARL methods in terms of sample efficiency in both cooperative tasks and competitive tasks. 

%% file: chapters/relatedwork.tex
\section{Related Work}\label{sec: related-work}

There are two important training paradigms in MARL, i.e., centralized training with decentralized execution (CTDE) and fully decentralized training. Typical algorithms using CTDE paradigm are value function decomposition algorithms \cite{whiteson2018qmix}, which deal with the non-stationarity and credit-assignment challenges by a centralized value function. Fully decentralized MARL algorithms include consensus learning over networked agents and communicating agents \cite{zhang2018fully,qu2019value}. In decentralized scenarios, where the agents can communicate with each other,
\cites{iqbal2019actor} proposed a method that integrates the received messages to learn a value function and a policy. In this paper, we propose a fully decentralized training method with a communication protocol.

Opponent modeling \cite{he2016opponent} is a feasible solution to the non-stationarity problem in MARL, which models others' policies with interaction experience. Behavior cloning is a straightforward method to learn opponents' policies \cite{lowe2017multi}. More advances include recursive reasoning with variational inference \cite{wen2019probabilistic}, maximum-entropy objective \cite{tian2019regularized}, and modeling the learning process of opponents \cite{foerster2018learning} etc.

There are generally two major problems for model-based RL methods, i.e., model learning and model usage.
For model learning, the most common solution is supervised learning \cite{nagabandi2018neural}, or non-parametric methods such as Gaussian processes \cite{kamthe2018data}. 
For model usage, a policy can be derived by exploiting the model with different algorithms such as Dyna \cite{sutton1990integrated}, shooting methods \cite{chua2018deep} and policy search with backpropagation through paths \cite{clavera2020model}. We refer to \cites{wang2019benchmarking} for details of model-based RL.
Theoretical analysis of model-based RL in single-agent scenarios has been investigated in recent literature. 
To name a few,
\cites{luo2019algorithmic} provided a monotonic improvement guarantee by enforcing a distance constraint between the learning policy and the data-collecting policy.
\cites{janner2019trust} derived a return discrepancy bound with the branched rollout, which studies how the model generalization affects the model usage.

For model-based MARL, there are relatively limited works of literature, to our knowledge. \cites{park2019multi} proposed to use a centralized auxiliary prediction network to model the environment dynamics to alleviate the non-stationary dynamics problem. \cites{kamra2020multi} and \cites{li2020evolvegraph} proposed interaction graph-based trajectory prediction methods, without considering policies. \cites{krupnik2019multi} built a centralized multi-step generative model with a disentangled variational auto-encoder to predict the environment dynamics and the opponent actions and then performed trajectory planning.
Unlike previous work, in our proposed \methodname{} each agent builds its environment model consisting of a dynamics model and opponent models, then learns its policy using the model rollouts. 
To our knowledge, \methodname{} is the first Dyna-style method in MARL.
More importantly, \methodname{} is supported by the theoretical bound of return discrepancy and convergence guarantee, which provides a principle to design further Dyna-style model-based MARL methods.

%% file: chapters/background.tex
\section{Problem Definition and Sample Complexity}
\label{sec: background}

\subsection{Problem Definition} 
We formulate the MARL problem as a stochastic game \cite{shapley1953stochastic}. An $n$-agent stochastic game can be defined as a tuple $(\mathcal{S}, \mathcal{A}^{1}, \ldots, \mathcal{A}^{n}, R^{1}, \ldots, R^{n}, \mathcal{T}, \gamma)$, where $\mathcal{S}$ is the state space of the stochastic game, $\mathcal{A}^{i}$ is the action space of agent $i\in \{1,\ldots ,n\}$, $\mathcal{A}=\Pi_{i=1}^n\mathcal{A}^{i}$ is the joint action space,
$R^{i}: \mathcal{S}\times \mathcal{A} \mapsto \mathbb{R}$ is the reward function of agent $i$, 
and $\gamma$ is the discount factor. State transition proceeds according to the dynamics function $\mathcal{T}: \mathcal{S}\times \mathcal{A} \mapsto \mathcal{S}$. 
For agent $i$, denote its policy as $\pi^{i}: \mathcal{S} \rightarrow \Delta(\mathcal{A}^i)$, which is a probability distribution over its action space, and $\pi^i(a_t|s_t)$ is the probability of taking action $a_t$ at state $s_t$. 
By denoting other agents' actions as $a^{-i}=\{a^{j}\}_{j\neq i}$, we can formulate the joint policy of other agents' as $\pi^{-i}(a_t^{-i}|s_t)=\Pi_{j\in\{-i\}}\pi^j(a_t^j|s_t)$. 
At each timestep, actions are taken simultaneously. Each agent $i$ aims at finding its optimal policy to maximize the expected return (cumulative reward), defined as
\begin{small}
\begin{equation*}
    \max_{\pi^i} \eta_i[\pi^{i}, \pi^{-i}] = \mathbb{E}_{(s_t,  a_t^{i}, a_{t}^{-i}) \sim \mathcal{T},\pi^{i},\pi^{-i}}\Big[\sum_{t=1}^\infty \gamma^{t}R^{i}(s_t, a_t^{i}, a_{t}^{-i}) \Big].
\end{equation*}
\end{small}


We consider a general scenario that each ego agent $i$ has no knowledge of other agents' policies, but can observe the histories of other agents,
i.e., $\{s_{1:t-1}, a_{1:t-1}^{-i}\}$ at timestep $t$. 
This scenario is common in decentralized MARL, such as in \cites{zhang2018fully} and \cites{qu2019value}.





\subsection{Two Parts of Sample Complexity}

From the perspective of the ego agent $i$, the state transition in MARL involves sub-processes, 
i.e., the action sampling process $\pi^{-i}(a^{-i}|s)$ of opponent agents given the current state $s$, and the transition process to the next state $\mathcal{T}(s^{\prime}|s,a^i,a^{-i})$ given the current state $s$ and the joint action $(a^{i}, a^{-i})$. 

\begin{figure}[b]
 \centering
 \includegraphics[width=0.85\linewidth]{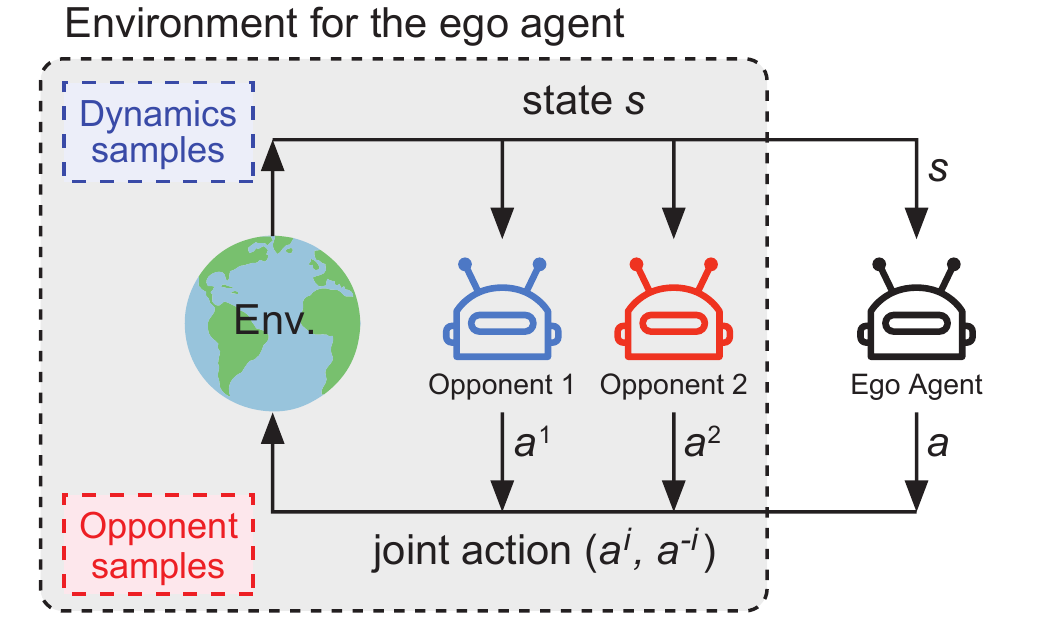}
 \caption{Two parts of the sample complexity in MARL.}
 \label{fig:marl-complexity}
\end{figure}

As such, for an ego agent to achieve some policy performance in a multi-agent environment, the sample complexity is specified in two parts: (i) \textit{dynamics sample complexity}, i.e., the number of state transition interactions between the agent group and the dynamics environment, and (ii) \textit{opponent sample complexity}, i.e., the total number of opponent action sampling given the conditioned state. Figure~\ref{fig:marl-complexity} illustrates such two parts of sample complexity.

We argue that it is necessary to study such two parts of sample complexity explicitly since the two sub-processes' approximations meet different challenges.
Environment dynamics modeling is usually computationally expensive as the state representation and the joint actions are usually of high dimensions.
Opponent modeling usually aims at approximating the opponents' policies and eases MARL training in a decentralized manner \cite{tian2020learning}. 



\paragraph{Opponent Sample Complexity as Communication.}
If we leverage model-based techniques derived from SARL in multi-agent cases with decentralized setting, we may need to call real opponents for the action given the current state, which can be formulated as a communication protocol.
For example, the agents communicate via calling for the opponent actions \cite{tian2020learning}: the ego agent sends a state (no matter real or simulated) to an opponent agent and receives the action sampled from the opponent's policy. As such, 
we can regard the opponent sample complexity as communication load.
To lower the opponent sample complexity, we replace the real opponents with learned opponent models in data simulation, which is analogous to selectively call some (or none) of opponents for useful information to reduce the communication load (or bandwidth) in multi-agent interactions
\cite{ryu2020multi}.




%% file: chapters/bound.tex
\section{Return Discrepancy Bounds}
\label{sec: bound}

In this section, we conduct theoretical analysis to better understand model-based MARL in the decentralized setting, where each ego agent learns a dynamics model and opponent models based on observed data. For a comprehensive analysis, we also discuss the bound in centralized model-based MARL in \app{app: cen_bounds}.



Usually, the dynamics model and opponent models have prediction errors, known as generalization errors. We investigate the influence of the learned models over the agent's performance. From the perspective of agent $i$, we aim to bound the discrepancy between the expected return $\eta_i[\pi^{i}, \pi^{-i}]$ of running the agent's policy $\pi^{i}$ in the real dynamics $\mathcal{T}$ with real opponents $\pi^{-i}$ and the expected return $\hat{\eta}_i[\pi^{i}, \hat{\pi}^{-i}]$ on a learned dynamics model $\hat{\mathcal{T}}(s^{\prime}|s,a^{i}, a^{-i})$ with opponent models $\hat{\pi}^{-i}$. The return discrepancy upper bound $C$ can be expressed as 
\begin{equation}
    \left|\eta_i[\pi^{i}, \pi^{-i}] - \hat{\eta}_i[\pi^{i}, \hat{\pi}^{-i}] \right| \leq C~.
    \label{eq: bound0}
\end{equation}
Once the upper bound $C$ is derived, it may indicate the key influence from the dynamics model and the opponent models, and thus helps design model-based algorithms accordingly. 



In MARL with iterative policy optimization, we denote data-collecting policies in the previous iterations as $(\pi_{D}^{i}, \pi_{D}^{-i})$. Then we measure the distribution shift between the current policies and the data-collecting policies with the total variation distance $\epsilon_{\pi}^i=\max_{s}D_{TV}(\pi^i(\cdot|s)\|\pi_D^i(\cdot|s))$, and 
similarly for the opponent agent $j\in \{-i\}$.
Besides, interacting with $(\pi_{D}^{i}, \pi_{D}^{-i})$, we denote the generalization error of the learned dynamics model $\hat{\mathcal{T}}$ compared to the real environment dynamics $\mathcal{T}$ as $\epsilon_m = \max_t\mathbb{E}_{(s_{t},a_{t}^{i},a_{t}^{-i})\sim\pi_{D}^{i}, \pi_{D}^{-i}}$ 
$[D_{TV}(\mathcal{T}(\cdot|s_{t},a_{t}^i,a_{t}^{-i})\|\hat{\mathcal{T}}(\cdot|s_{t},a_{t}^i,a_{t}^{-i}))]$ and the error of the opponent model for agent $j$ as $\epsilon_{\hat{\pi}}^j=\max_{s}D_{TV}(\pi^{j}(\cdot|s)\| $ $\hat{\pi}^{j}(\cdot|s))$. Now we are ready to provide an upper bound.
\begin{restatable}{theorem}{bound}
Assume that the error of the learned dynamics model is bounded at each timestep by $\epsilon_m$, and the distance of the policies are bounded by $\epsilon_{\pi}^i$, $\epsilon_{\pi}^{j}$, and the errors of the opponent models are bounded by $\epsilon_{\hat{\pi}}^j$, for $j \in \{-i\}$. The return discrepancy upper bound can be expressed as
\begin{small}
\begin{align}
    &~ \big| \eta_i[\pi^{i}, \pi^{-i}] - \hat{\eta}_i[\pi^{i}, \hat{\pi}^{-i}] \big| \nonumber \\
    & \leq 2r_{\text{max}}\Big[\frac{\gamma(\epsilon_m+2\epsilon_{\pi}^i+2\sum_{j \in \{-i\}}\epsilon_{\pi}^{j}+\sum_{j \in \{-i\}}\epsilon_{\hat{\pi}}^j)}{1-\gamma}\\
    & \quad\quad\quad\quad+2\epsilon^{i}_{\pi}+2\sum_{j \in \{-i\}}\epsilon_{\pi}^{j}+\sum_{j \in \{-i\}}\epsilon_{\hat{\pi}}^j\Big].\nonumber
\end{align}
\end{small}
\end{restatable}

\begin{proof} 
The proof is provided in \app{sec:multi-opponent-bound}. 
\end{proof}
This upper bound can be minimized by improving models' accuracy, which is the standard objective in previous literature. We investigate the model usage instead, and a more instructive and practical upper bound is needed.



Here we first investigate a $k$-step rollout scheme, which is a natural extension of MBPO \cite{janner2019trust} to multi-agent scenario. This rollout scheme begins from a state collected by the previous data-collecting policies, and then run the current policy for $k$ steps under the learned dynamics model and the opponent models. 
Denoting the return of using multi-agent branched rollouts as $\eta_i^{\text{branch}}$, and the bound of the generalization error of the learned dynamics model with the current policies $(\pi^{i}, \pi^{-i})$
as $\epsilon_m^{\prime}=$
$\max_t\mathbb{E}_{(s_{t},a_{t}^{i},a_{t}^{-i})\sim\pi^{i}, \pi^{-i}}[D_{TV}$ $(\mathcal{T}(\cdot|s_{t},a_{t}^i,a_{t}^{-i})\|\hat{\mathcal{T}}(\cdot|s_{t},a_{t}^i,a_{t}^{-i}))]~,$ we derive the following return discrepancy upper bound.

\begin{theorem}\label{th: discrepancy-branched-rollout}
Assume the generalization error is bounded by $\epsilon_m^{\prime}$ and the distance of the policies are bounded as $\epsilon_{\pi}^{i}$, $\epsilon_{\pi}^{j}$ and $\epsilon_{\hat{\pi}}^j$ for $j \in \{-i\}$. The bound of the discrepancy between the return in real environment $\eta_i[\pi^{i}, \pi^{-i}]$ and the return using the dynamics model and opponent models with branched rollouts $\eta_i^{\text{branch}}[(\pi^{1}_{D},\hat{\pi}^1),\ldots,(\pi^{i}_{D},\pi^{i}),\ldots,(\pi^{n}_{D},\hat{\pi}^{n})]$ is
\begin{small}
\begin{align}
        \big| \eta_i&[\pi^{i}, \pi^{-i}] - \eta_i^{\text{branch}}[(\pi^{1}_{D},\hat{\pi}^1),\ldots,(\pi^{i}_{D},\pi^{i}),\ldots,(\pi^{n}_{D},\hat{\pi}^{n})] \big| \nonumber \\
        & \leq 2r_{\text{max}}\Big[\underbrace{k\epsilon_m^{\prime}+ (k+1)\sum_{j \in \{-i\}}\epsilon_{\hat{\pi}}^j}_{\text{model generalization error}} \nonumber\\
        & \quad+ \underbrace{\gamma^{k+1}\big(\epsilon_{\pi}^{i}+\sum_{j \in \{-i\}}\epsilon_{\pi}^{j}\big)+\frac{\gamma^{k+1}(\epsilon_{\pi}^{i}+\sum_{j \in \{-i\}}\epsilon_{\pi}^{j})}{1-\gamma}}_{\text{policy distribution shift}} \Big] \nonumber \\
        & = C(\epsilon_m^{\prime}, \epsilon_{\pi}^{i}, \epsilon_{\pi}^{-i}, \epsilon_{\hat{\pi}}^{-i}, k)~,
    \label{eq: bound2}
\end{align}
\end{small}
where 
$\epsilon_{\pi}^{-i}=\sum_{j \in \{-i\}}\epsilon_{\pi}^{j}$, $\epsilon_{\hat{\pi}}^{-i}=\sum_{j \in \{-i\}}\epsilon_{\hat{\pi}}^j$ and the pair $(\pi_{D}^j,\hat{\pi}^j)$ means that the policy $\pi_{D}^j$ and opponent model $\hat{\pi}^j$ are used before and after the branch point respectively for agent $j$.
\end{theorem} 
\begin{proof}
The proof is provided in \app{sec:multi-opponent-bound}.
\end{proof}

The upper bound $C(\epsilon_m^{\prime}, \epsilon_{\pi}^{i}, \epsilon_{\pi}^{-i}, \epsilon_{\hat{\pi}}^{-i}, k)$ 
in \eq{eq: bound2} consists of the generalization errors of both dynamics model and opponent models as well as the policy distribution shift. Intuitively, choosing the optimal $k^{*}=\arg\min_{k>0} C(\epsilon_m^{\prime}, \epsilon_{\pi}^{i}, \epsilon_{\pi}^{-i}, \epsilon_{\hat{\pi}}^{-i}, k)$ 
with sufficiently low $\epsilon_m^{\prime}$ and $\sum_{j \in \{-i\}}\epsilon_{\hat{\pi}}^j$ minimizes the discrepancy, which, however, cannot directly achieve a low discrepancy if any opponent model has relatively large error. So we need to reduce the upper bound such that the policy trained with the model rollouts will still perform well in real environment, leading to improved sample efficiency. To be more specific, we can optimize the target $\eta_{i}$ through optimizing $\eta_{i}^{\text{branch}}$ if the bound is tight. And improving $\eta_{i}^{\text{branch}}$ by $C(\epsilon_m^{\prime}, \epsilon_{\pi}^{i}, \epsilon_{\pi}^{-i}, \epsilon_{\hat{\pi}}^{-i}, k)$ guarantees to improve the target $\eta_{i}$. It means that the policy improved under the environment model and the opponent models will get improved performance in the real environment. 

\minisection{Remark}
As the policy is trained using mainly the samples from the model rollouts, which do not contribute to the sample complexity. 
A reduced return discrepancy upper bound indicates that the policy will get more improvement in the real environment given that the policy is improved to the same degree using the models. In other words, the policy will obtain the same performance improvement in the real environment with fewer samples with a reduced discrepancy upper bound.

Now the problem becomes how to reduce the bound in order to achieve low sample complexity. Considering that $\epsilon_{\hat{\pi}}^{j}$ may be different across different opponent models, and sometimes the ego agent can call the real opponent agents for real actions, the rollout length $k$ needs not to be the same across different opponents. Thus, in Section~\ref{sec: method} we propose a method called \emph{\fullmethodname{}} (\methodname{}) to reduce the upper bound.




%% file: chapters/method.tex
\section{The \methodname{} Method}
\label{sec: method}

Based on the bound analysis in Theorem~\ref{th: discrepancy-branched-rollout}, now we present the detailed \methodname{} algorithm and prove its convergence.

\begin{algorithm}[t]
\caption {\methodname{} Algorithm}
\label{alg}
Initialize policy $\pi_{\zeta}$, Q value function $Q_{\omega}$, predictive model $\hat{\mathcal{T}}_{\theta}$, opponent models $\pi_{\phi^{j}}$ for $j\in\{-i\}$, environment dataset $\mathcal{D}_{\text{env}}$, model dataset $\mathcal{D}_{\text{model}}$. \\ 
\For{$N$ epochs}
{
     Train model $\hat{\mathcal{T}}_{\theta}$ on $\mathcal{D}_{\text{env}}$. \label{alg: learn_model} \\
    \For{$E$ steps}
    {
        Take actions in environment via $\pi_{\zeta}$ with real opponents, add the transitions to $\mathcal{D}_{\text{env}}$.\\
        Train all opponent models $\pi_{\phi^{j}}$. \label{alg: learn_opp} \\
        Compute the errors for each opponent model $\epsilon_{\hat{\pi}}^{j}$. \label{alg: opp_err}\\
        \vspace{-5pt}
        For each opponent, compute $n^j = \lfloor k\frac{\min_{j^{\prime}}\epsilon_{\hat{\pi}}^{j^{\prime}}}{\epsilon_{\hat{\pi}}^j} \rfloor$.\\
        \For {$M$ model rollouts}
        {
            Sample $s_t$ uniformly from $\mathcal{D}_{\text{env}}$.\\
            Perform $k$-step rollouts from $s_t$:\\
            \For {$p=1, \ldots, k$}  
            {
            \label{alg_model_rollut_start}
            $a_{t+p-1}^{i} = \pi_{\zeta}(s_{t+p-1})$ \\
            For each opponent $j$: \\
                \If {$p\leq n^j$}
                {
                    $a_{t+p-1}^{j} = \pi_{\phi^{j}}(s_{t+p-1})$
                }
                \Else
                {
                    \SetKwFunction{FComm}{Comm}
                    $a_{t+p-1}^{j}$ = \FComm($s_{t+p-1}$, $j$)
                    \label{send_and_get}
                }
            $s_{t+p} = \hat{\mathcal{T}}_{\theta}(s_{t+p-1}, a_{t+p-1}^{i}, a_{t+p-1}^{-i})$
            }
            \label{alg_model_rollut_end}
            Add the transitions to $\mathcal{D}_{\text{model}}$.
        }
        Train the Q function and the policy using data from $\mathcal{D}_{\text{model}}$ \label{alg: learn_agent} with the loss in \eq{eq: Q} and \eq{eq: pi}.
    }
}
\SetKwProg{Fn}{Function}{ :}{}
\Fn{\FComm ($s$, $j$)}{
        Send state $s$ to agent $j$. \\
        \KwRet the received $a^{j}$
}
\end{algorithm}

\subsection{Algorithm Details of \methodname{}}
\label{subs: method}
We propose a model-based MARL algorithm named Multi-Agent Branched-Rollout Policy Optimization (\methodname{}). The overall algorithm of \methodname{} is shown in \alg{alg}. For simplicity, the detailed algorithm is described in the perspective of agent $i$. \methodname{} includes some key components, and the implementation of them is based on previous work, which serve as the preliminaries and are described in \app{app:building-blocks}. 

In \methodname{}, the agent $i$ learns a dynamics model and an opponent model for each of other agents and learns a policy based on the data collected from the model rollouts.


\begin{figure*}[t]
    \centering
    \hspace{-10pt}\includegraphics[width=0.85\linewidth]{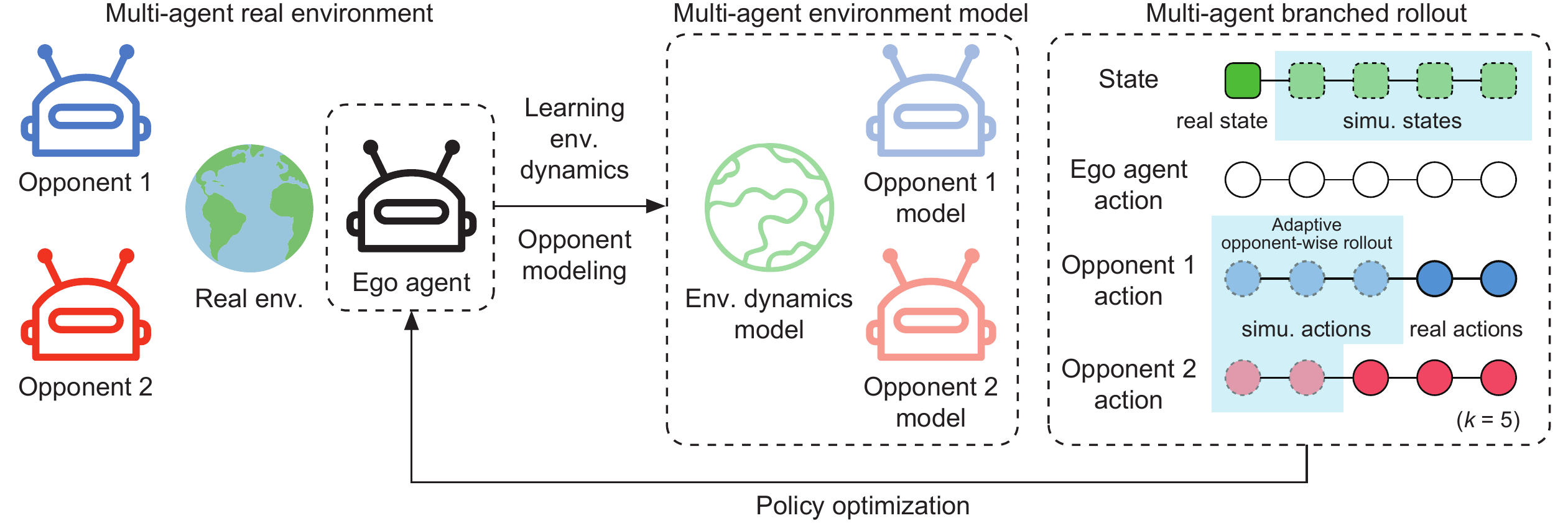}
    \caption{Illustration of the \textit{adaptive opponent-wise rollout policy optimization} (\methodname) method from the perspective of the ego agent. The ego agent begins the rollouts from the states collected from the real environment interacting with real opponents, then runs $k$ steps under the learned dynamics model. For each opponent agent, the ego agent interacts with it directly or with the opponent model for it in the $k$-step rollout, according to the prediction performance of its opponent model.}
    \label{fig: ma-mbrl}
    \vspace{-12pt}
\end{figure*}

\subsubsection{Agent Learning}
For the dynamics model, in line \ref{alg: learn_model} of \alg{alg}, a bootstrap ensemble of probabilistic dynamics models is trained to predict the environment dynamics as in \cites{chua2018deep}, and we select one of the models randomly from a uniform distribution when generating a prediction. For the opponent model, in line \ref{alg: learn_opp} of \alg{alg}, the policy of one opponent agent $j$ can be modeled as a Gaussian distribution of the action. We further use the opponent models to encourage the coordination behaviors. Agent $i$ makes a decision based on both the current state and the inferred opponent actions $\hat{a}^{-i}$.

For the policy learning, \methodname{} is implemented based on the Multi-Agent Soft Actor-Critic (MASAC) algorithm, which is a multi-agent version of Soft Actor-Critic algorithm \cite{haarnoja2018improved}. In line \ref{alg: learn_agent} of \alg{alg}, agent $i$ alternates between optimizing its policy and updating the value function. With the reparameterization function $f$ and a Gaussian $\mathcal{N}$, the policy optimization objective is
\begin{small}
\begin{align}
\label{eq: pi}
 J(\pi)=\mathbb{E}_{s_t, a_t^{-i} \sim \mathcal{D}, \epsilon_t \sim \mathcal{N}}[
 &\alpha\log\pi(f(\epsilon_t;s_t)|s_t, \hat{a}_{t}^{-i}) \\
 & -Q(s_t, f(\epsilon_t;s_t), a^{-i}_t)]~,\nonumber  
\end{align}
\end{small}
where $\mathcal{D}$ is the replay buffer, $\epsilon_t$ is the input noise and $\hat{a}_{t}^{-i} \sim \hat{\pi}^{-i}(\cdot|s_t)$.

The loss function of the Q function is
\begin{small}
\begin{align}
\label{eq: Q}
\vspace{-5pt}
        & J(Q) =  \mathbb{E}_{s_t, a_t^{i}, a_t^{-i} \sim \mathcal{D}}\Big[
        \frac{1}{2}\Big(Q(s_t,a_t^{i}, a_t^{-i}) \\
        & - \big(r_t + \gamma Q(s_{t+1}, a_{t+1}^{i}, a_{t+1}^{-i})
        -\alpha \log \pi(a_{t+1}^{i} | s_{t+1})\big)\Big)^2\Big]~, \nonumber
\end{align}
\end{small}
where $\alpha$ is the temperature hyperparameter.

\subsubsection{Adaptive Opponent-wise Rollout}
In this work, a $k$-step model rollout begins from a real state collected in the real environment. We design the \textit{adaptive opponent-wise rollout} scheme to reduce the bound. The opponent model generalization error term $(k+1)\sum_{j \in \{-i\}}\epsilon_{\hat{\pi}}^j$ of the bound in \eq{eq: bound2} reveals that different magnitudes of opponent model errors may lead to different contributions to the compounding error of the multi-step simulated rollouts. Intuitively, if too short rollouts are performed, the relatively accurate opponent models are not fully utilized, leading to low sample efficiency. In contrast, if the rollout is too long, the relatively inaccurate opponent models may cause the rollouts to depart from the real trajectory distribution heavily, leading to degraded performance in the environment and thereby low sample efficiency.
Thus our proposed adaptive opponent-wise rollout scheme lowers the rollout length of the relatively inaccurate opponent models but keeps long rollout length for relatively accurate opponent models, as shown in \fig{fig: ma-mbrl}.

We can achieve this 
since the agent $i$ can play with either the learned opponent model $\hat{\pi}^{j}$ or the real policy $\pi^{j}$ of each opponent $j$. In detail, for the opponent $j$, $\hat{\pi}^{j}$ is used for the first $n^j = \lfloor k\frac{\min_{j^{\prime}}\epsilon_{\hat{\pi}}^{j^{\prime}}}{\epsilon_{\hat{\pi}}^j} \rfloor$ steps, then the agent $i$ interacts with the real opponent $\pi^{j}$ in the rest $k-n^j$ steps. 
The compounding error contribution of each opponent model is thus bounded by $k \min_{j^{\prime}}{\epsilon_{\hat{\pi}}^{j'}}$. Note that the adaptive opponent-wise rollout scheme requires the ego agent $i$ can obtain the actions $a_t^j$ from the real policy in the last $k-n^j$ steps. In \cref{send_and_get} of \alg{alg}, the ego agent $i$ sends the simulated states to the opponent $j$ and requires the response, following the predefined communication protocol. One may argue that it will introduce extra opponent sample complexity since we need to communicate with the real opponent in the model rollouts. However, if we only use the opponent models, more errors are exploited, and it will lead to a poor policy and, finally, the sample complexity to achieve some good performance will be high. Our experiments in Section \ref{sec: exp} will justify this claim.


With such $n^j$'s, the generalization error caused by the opponent models becomes $\sum_{j \in \{-i\}}(n^j+1)\epsilon_{\hat{\pi}}^j \simeq (n-1) k \min_{j^{\prime}}\epsilon_{\hat{\pi}}^{j^{\prime}}$, which is remarkably reduced and makes a good balance of the contribution to overall generalization error from different opponent models. 
According to the remark of \Cref{th: discrepancy-branched-rollout}, improving the surrogate return $\eta_{i}^{\text{branch}}$ with a tighter discrepancy bound will improve the target return $\eta_{i}$ more efficiently.
Note that the comparison of different model usage schemes is provided in \app{app:two-rollout}, where the above scheme yields the highest asymptotic performance and sample efficiency.


\subsection{Convergence Guarantee}
\label{subs: convergence}


According to \cites{wei2018multiagent}, the optimal policy learned by MASQL is $\pi^*_{\text{MASQL}} = \exp{( \frac{1}{\alpha}Q^*(s_t, \cdot) - V^*(s_t) )} = \exp{( \frac{1}{\alpha}Q^*(s_t, \cdot))} / \exp{\left(V^*(s_t)\right)}$, where $Q^*$ and $V^*$ are the optimal \emph{Q} function and state value function respectively. As \cites{haarnoja2018improved} showed, the optimal policy learned by MASAC should be $\pi^*_{\text{MASAC}} = \exp{( \frac{1}{\alpha}Q^*(s_t, \cdot))} / Z(s_t)$. Since the partition function is $Z(s_t)=\exp{\left(V^*(s_t)\right)}$, with given the same optimal $Q^*$ and $V^*$, MASQL is equivalent to MASAC from the perspective of policy learning. With this fact, we prove that (1) using the learned dynamics model and the opponent models in \methodname{} still guarantees the convergence of MASQL; (2)  MASQL still guarantees the convergence of Nash Q-learning \cite{hu2003nash}. Thus we prove that \methodname{} achieves the convergence solution of Nash Q-learning. The theorems and proofs can be found in  \app{app: converge}, \Cref{th: MABRPO_convergence} and \Cref{th: MASQL_convergence}.

The theorems guarantee \methodname{}'s convergence under several assumptions. However, we show that the assumptions are not necessary conditions for the learning algorithm to converges by the following experimental results in \Cref{fig: climb}, in which the convergence is still guaranteed when the assumptions are not strictly satisfied. More empirical findings can be found in \cites{yang2018mean}.

%% file: chapters/experiment.tex
\section{Experiments}\label{sec: exp}

\begin{figure*}
	\centering
	\includegraphics[width=0.9\linewidth]{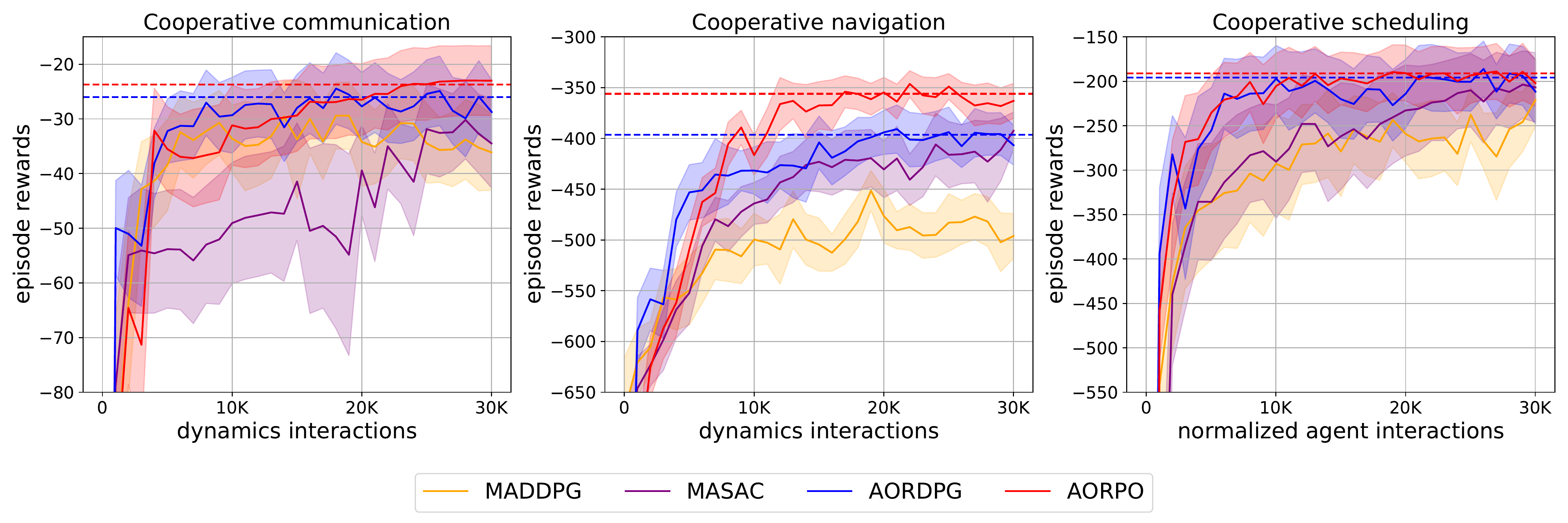}
	\caption{The average episode reward performance of five different random seeds in the cooperative tasks. \textit{Dynamics interactions} means the episodes in the training stage (the agents interacts with the dynamics environment 25 times in an episode). \textit{Normalized agent interactions} means the times the agent groups complete a full interaction, which is divided by 25. 
	}
	\label{fig: cooperative}
\end{figure*}


We evaluate the performance of \methodname{} in both competitive and cooperative tasks\footnote{Our implementation is available at \url{https://github.com/apexrl/AORPO}.}. We demonstrate that \methodname{} can converge to the Nash equilibrium and that \methodname{} is more sample-efficient than a series of baselines in both competitive and cooperative tasks. See \app{app: exp} for task details. 

\minisection{Compared Methods} We compare our method \methodname{} with the strong model-free MARL algorithms, MADDPG \cite{lowe2017multi} and MASAC. To show that our multi-agent models approach is generally helpful and to have fair comparisons, we also implement \methodname{} on top of MADDPG, named as \methodnametwo{}, for more comprehensive analysis.

\minisection{Implementation} As for the practical implementation of \methodname{}, we first collect real experience data using model-free algorithms and use them to pre-train the dynamics model and the opponent models. 
Other implementation details, including network architectures and important hyperparameters, are provided in \app{app: hyper_set}.


\subsection{Competitive Tasks}

\textbf{Climb} is a modified \textit{Battle of Sex} game. Agent 1 picks a number $a \in [-1, 1]$, and agent 2 picks a number $b \in [-1, 1]$. Their actions form a position $(a, b)$. There are 2 states and 2 corresponding landmarks located in the lower-left $(-0.5, -0.5)$ and in the upper right $(0.5, 0.5)$ respectively. The only Nash equilibrium is that the agents go to the lower-left landmark at state 1 and go to the upper right landmark at state 2.

\begin{figure}[tb]
	\centering
	\includegraphics[width=1\linewidth]{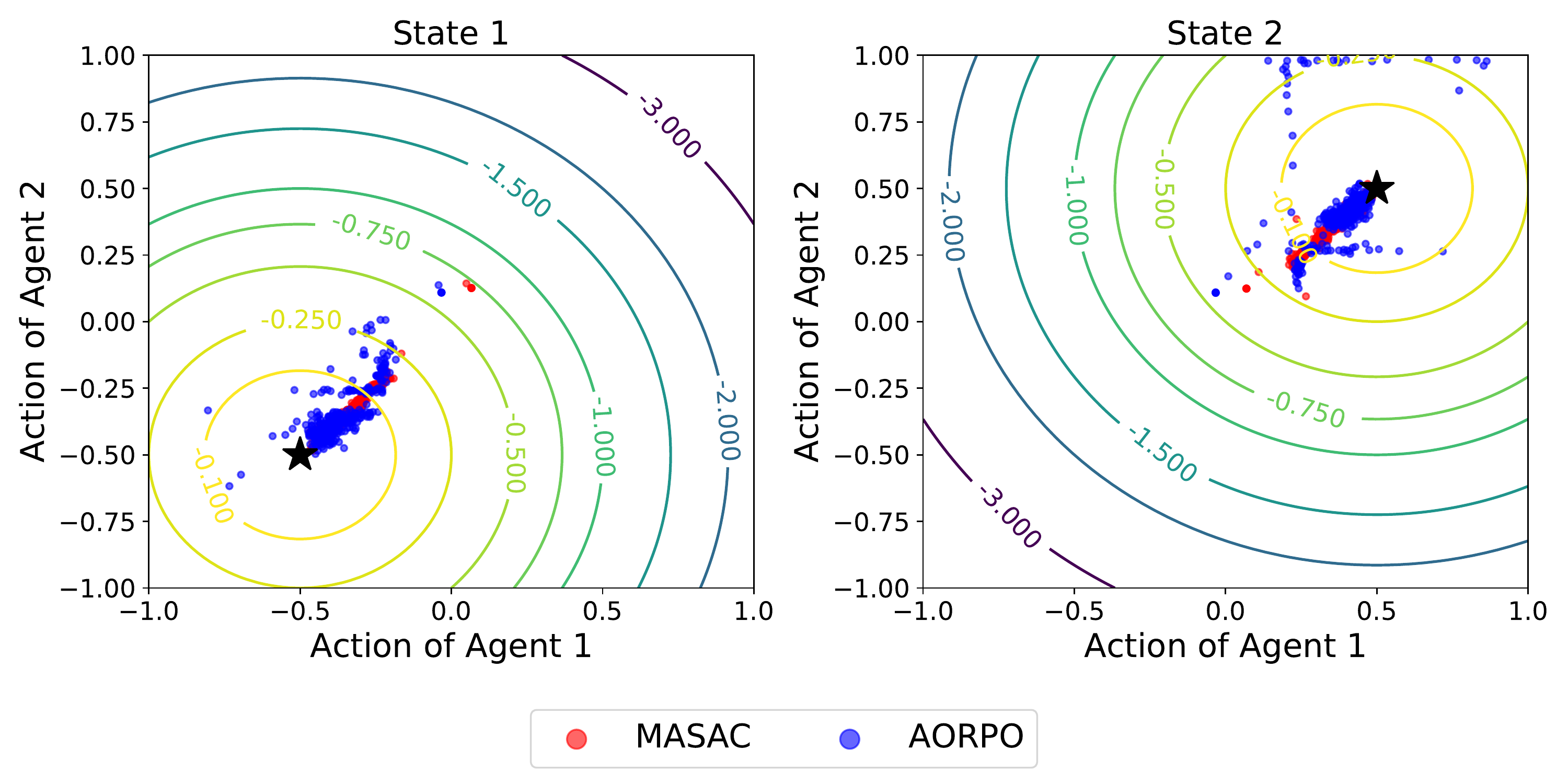}
	\caption{Climb task convergence.}
	\label{fig: climb}
\end{figure}


The reward surfaces of converged MASAC and \methodname{} for the two agents in the two states are shown in \fig{fig: climb}. The result verifies that \methodname{} will converge to the Nash equilibrium. In \app{subs: tab_res} and \ref{sec:comp-with-mbmarl}, we further provide more detailed results of comparing model-based and model-free methods conducted in several competitive environments of the Multi-Agent Particle Environment \cite{lowe2017multi} and the results of comparing \methodname{} with another model-based MARL, i.e., MAMSGM \cite{krupnik2019multi}.

\subsection{Cooperative Tasks}
Based on a multi-agent particle environment, we evaluate our method in two types of cooperative tasks: CESO tasks for complex environment dynamics and simple opponents, while SECO tasks for simple environment dynamics and complex opponents. CESO tasks include \textit{Cooperative communication} with two agents and three landmarks and \textit{Cooperative navigation} with three agents and three landmarks.
The studied SECO task is \textit{Cooperative scheduling} with three agents and three landmarks, as detailed in \app{app: exp}.

The laws of mechanics and the collisions' stochastic outcomes make difficulties for the dynamics model prediction, so we mainly consider the efficiency of interacting with the environment in the CESO tasks. In the left and middle subfigures of \fig{fig: cooperative}, both \methodname{} and \methodnametwo{} can reach the asymptotic performance of the state-of-the-art model-free baselines with fewer interactions with the environment, i.e., achieving lower dynamics sample complexity. 

The \textit{Cooperative scheduling} task is configured as a SECO task, with two kinds of agents and two kinds of action spaces. We consider the efficiency of interacting with the other agents in the SECO tasks.  As shown in the right subfigure of \fig{fig: cooperative}, 
our methods \methodname{} and \methodnametwo{} reach the asymptotic performance with fewer agents' interactions, which means the opponent sample complexity is lower. One may think using opponent models will indeed introduce extra agent interactions when running rollouts, but interestingly, \methodname{} uses fewer agent interactions to achieve some performance since it helps policy learn faster. More detailed experiments are presented in \app{app: exp}. 


\subsection{Analysis on Model Usage}
We further discuss the effect of the opponent model usage. Although the generalization errors caused by the opponent models in the return discrepancy are hard to estimate, the opponent models affect the performance through the model compounding errors. We investigate the model compounding errors when using adaptive opponent-wise rollout or using the opponent models in a whole rollout. Specifically, for a real trajectory $(s_0, a_0, \ldots, s_h)$, a branched rollout from the start state $s_0$ under learned dynamics model and opponent models is denoted as $(s_0, \hat{a}_0, \ldots, \hat{s}_h)$. Following \cites{nagabandi2018neural}, the model compounding error is defined as $\epsilon_{\text{c}} = \frac{1}{h}\sum_{i=1}^{h}\|s_i-\hat{s}_i\|^2_2$. 



Using opponent models reduces interactions among agents but causes extra compounding errors compared with real opponents. We show the trade-off between model errors and interaction times in \fig{fig: comp_r_l}, where the left and right y-axes are for the solid lines and the dotted lines, respectively. We observe the adaptive opponent-wise rollout achieves lower model compounding error than rollout with all opponent models while reducing the number of interactions compared to rollout with real opponents. 


\begin{figure}
	\includegraphics[width=\columnwidth]{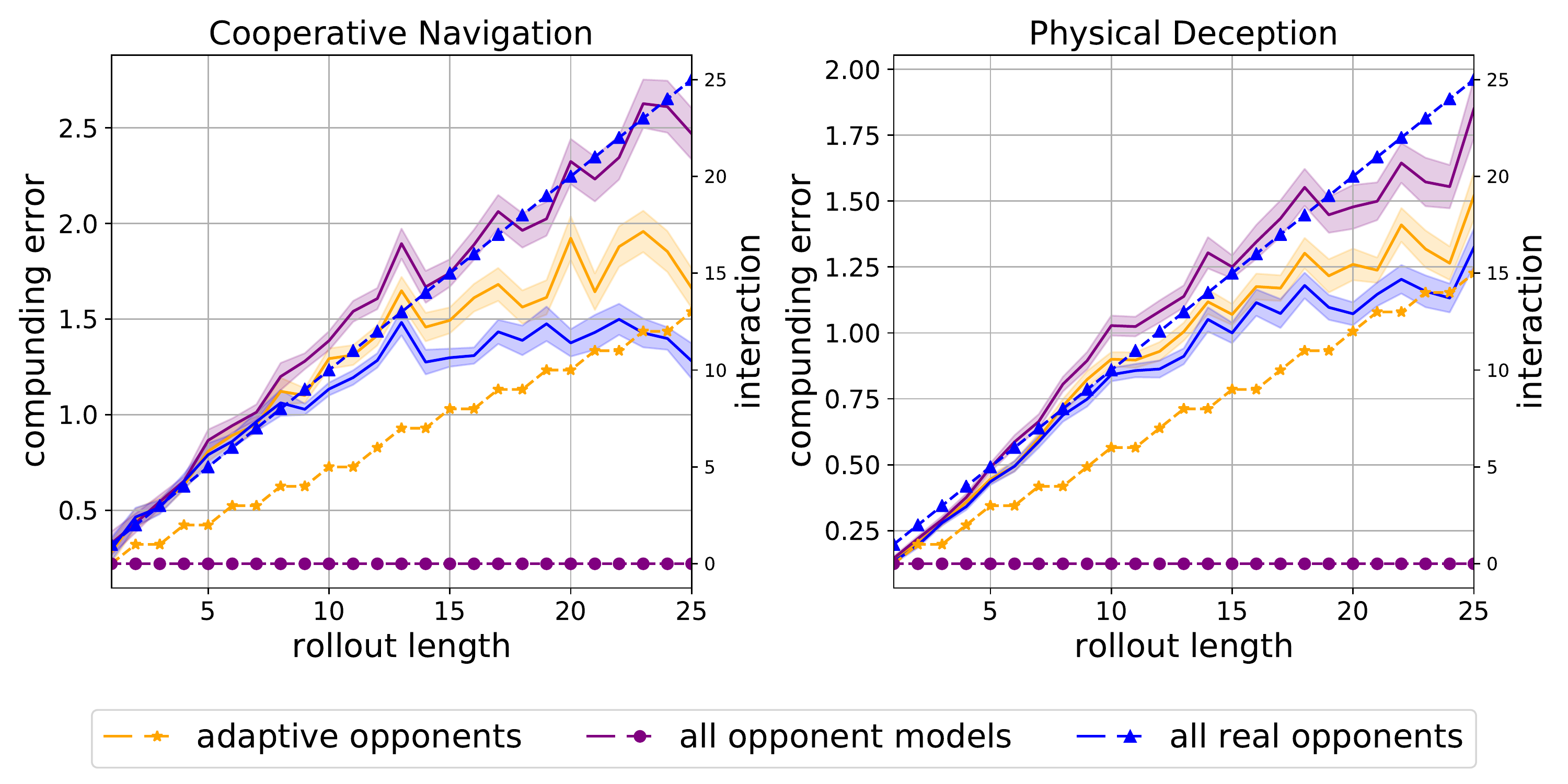}
	\captionof{figure}{Model compounding errors and interactions numbers of different model usages.}
	\label{fig: comp_r_l}
\end{figure}
\begin{figure}
	\includegraphics[width=\columnwidth]{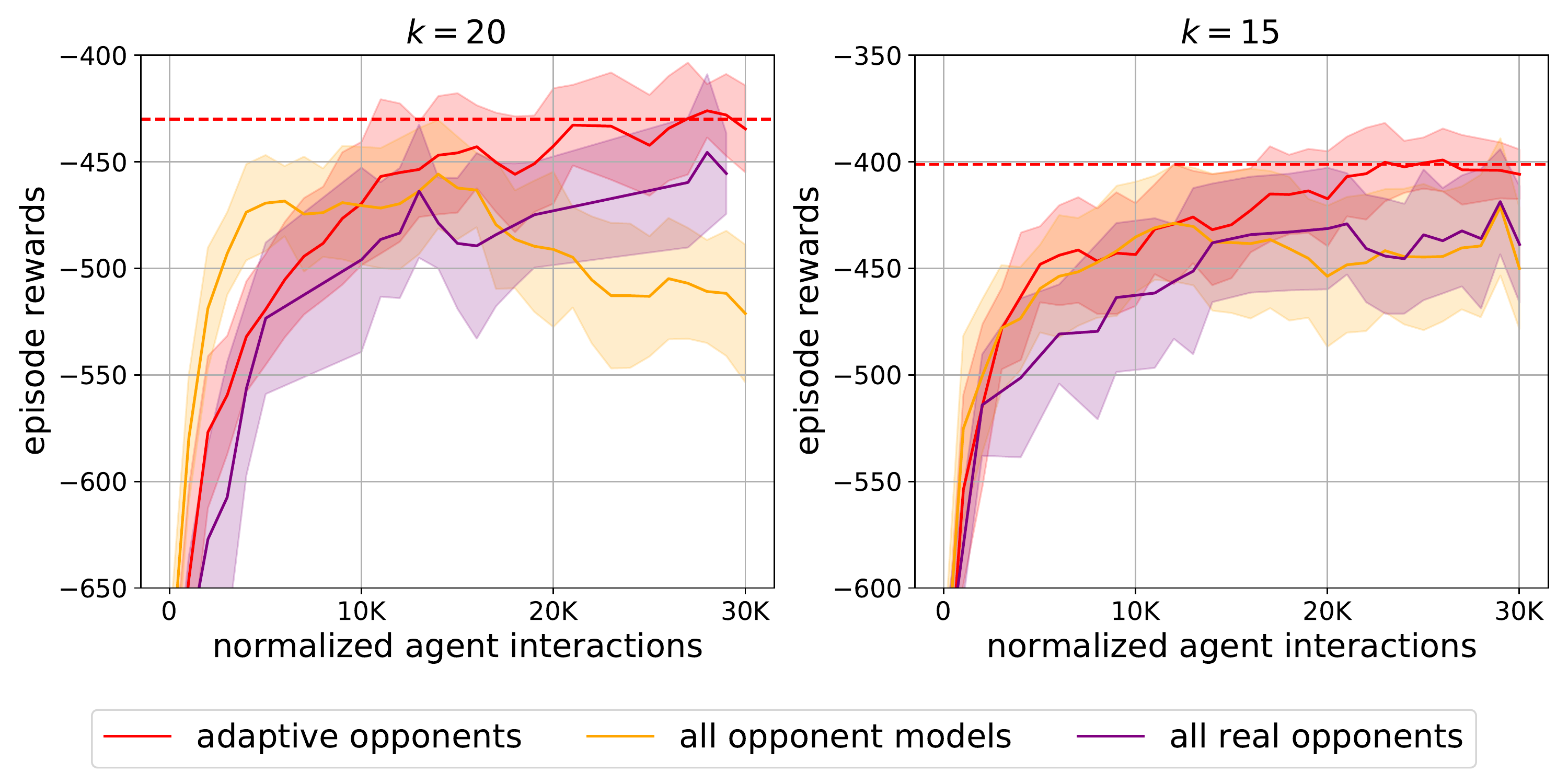}
	\captionof{figure}{Performance curves of average episode reward of different model usages.}
	\label{fig: ada}
\end{figure}

Moreover, we investigate how different rollout schemes affect the sample efficiency. In \fig{fig: ada}, we compare three kinds of usages in a Cooperative navigation scenario using two settings of rollout length $k$. We notice that using the adaptive opponent-wise rollout achieves the same performance with fewer interactions among the agents than interacting with all real opponents, verifying that although we introduce some sample complexity while simulating data, the reduced discrepancy bound ensures the model-based methods to improve the sample efficiency. 
It can also avoid performance drop caused by compounding error compared with using all opponent models, which means that the adaptive opponent-wise rollout scheme mitigates some bad opponent models' negative effect by leveraging different rollout lengths of opponent models.

%% file: chapters/conclusion.tex
\section{Conclusion}\label{sec:conclusion}
In this paper, we investigated model-based MARL problems with both theoretical and empirical analyses. We specified two parts of sample complexity in MARL and derived an upper bound of the return discrepancy in model-based MARL w.r.t. the policy distribution shift and the generalization errors of the learned dynamics model and opponent models.
Inspired by the theoretical analysis, we designed the \methodname{} algorithm framework, in which the return discrepancy can be reduced by adaptive opponent-wise rollout controlling and the two parts of sample complexity are strategically reduced by the learned dynamics and opponent models. We then proved that \methodname{} can converge to Nash Q-values with reasonable assumptions. In experiments, \methodname{} has shown highly comparable asymptotic performance with the model-free MARL baselines while achieving higher sample efficiency. 
For future work, we plan to look deeper into various adaptive opponent-wise rollout schemes for different settings of model-based MARL. 
We will also investigate theoretical and empirical results of multi-agent dynamics model learning to further ease model-based MARL and improve its sample efficiency.

%% file: chapters/appendix.tex
\newpage
\onecolumn
\appendix
\section*{Appendix}

\section{Upper Bounds for Centralized Model-based MARL}
\label{app: cen_bounds}

In this section, we derive an upper bound of the return discrepancy between the expected return under the real environment and the expected return under a learned dynamics model, both from the perspective of the ego agent, i.e., agent $i$, interacting with the known real opponent agents, which corresponds to centralized MARL settings and will not produce extra opponent sample complexity.

\begin{theorem}
	\label{th: cen_bound} Assume that the expected total variation distance between the learned dynamics model $\hat{\mathcal{T}}$ and the real dynamics model $\mathcal{T}$ is bounded at each timestep under the expectation of $\pi_D^i$ and $\pi_D^{-i}$ by $\epsilon_m=\max_t\mathbb{E}_{(s_{t},a_{t}^{i},a_{t}^{-i})\sim\pi_{D}^{i},\pi_{D}^{-i}}[D_{TV}$
	$(\mathcal{T}(\cdot|s_{t},a_{t}^i,a_{t}^{-i})\|\hat{\mathcal{T}}(\cdot|s_{t},a_{t}^i,a_{t}^{-i}))]$, and the distance of the policies are bounded as $\epsilon_{\pi}^i=\max_{s}D_{TV}(\pi^i(\cdot|s)\|\pi_D^i(\cdot|s))$, $\epsilon_{\pi}^{-i}=\max_{s}D_{TV}($ $\pi^{-i}(\cdot|s)\|\pi_D^{-i}(\cdot|s))$, where subscript $D$ means the certain policies are used to collect data\footnote{In our paper, $\mathbb{E}_{s_t, a_t\sim \pi}\left[\cdot\right]$ has the same meaning as $\mathbb{E}_{s, a\sim \pi_{t}}\left[\cdot\right]$ in \cites{janner2019trust}.}. Then the returns under the real dynamics and the learned dynamics model are bounded as:
	$$
	\left|\eta_i[\pi^{i}, \pi^{-i}] - \hat{\eta}_i[\pi^{i}, \pi^{-i}]\right| \leq 2r_{\text{max}} \Big[\frac{\gamma(\epsilon_m+2\epsilon_{\pi}^i+2\epsilon_{\pi}^{-i})}{1-\gamma}+2\epsilon^{i}_{\pi}+2\epsilon^{-i}_{\pi}\Big]~,
	$$
	where $\hat{\eta}$ means that the policies are executed under the learned dynamics model.
\end{theorem}

\begin{proof}
	We can use \Cref{le: bl3} to bound the returns, and it requires to bound the error of the learned dynamics model. Thus we need to introduce $\pi^{i}_{D}$ and $\pi^{-i}_{D}$ by adding and subtracting $\eta[\pi^{i}_{D},\pi^{-i}_{D}]$:
	$$
	\left|\eta_i[\pi^{i}, \pi^{-i}] - \hat{\eta}_i[\pi^{i}, \pi^{-i}]\right| \leq \underbrace{\left|\eta_i[\pi^{i}, \pi^{-i}] - \eta_i[\pi^{i}_{D},\pi^{-i}_{D}]\right|}_{L_1} + \underbrace{\left|\eta_i[\pi^{i}_{D},\pi^{-i}_{D}] - \hat{\eta}_i[\pi^{i}, \pi^{-i}]\right|}_{L_2}~.
	$$
	We can bound both $L_1$ and $L_2$ using \Cref{le: bl3}.
	
	For $L_1$, we apply \Cref{le: bl3} using $M=0$, $\mathcal{P}=\epsilon_{\pi}^{i}$, $O=\epsilon_{\pi}^{-i}$, and obtain
	$$
	L_1 \leq 2r_{\text{max}} \Big[\frac{\gamma(\epsilon_{\pi}^i+\epsilon_{\pi}^{-i})}{1-\gamma}+\epsilon^{i}_{\pi}+\epsilon^{-i}_{\pi} \Big]~.
	$$
	
	For $L_2$, we apply \Cref{le: bl3} using $M=\epsilon_m$, $P=\epsilon_{\pi}^{i}$, $O=\epsilon_{\pi}^{-i}$, and obtain
	$$
	L_2 \leq 2r_{\text{max}}\Big[\frac{\gamma(\epsilon_m+\epsilon_{\pi}^i+\epsilon_{\pi}^{-i})}{1-\gamma}+\epsilon^{i}_{\pi}+\epsilon^{-i}_{\pi}\Big]~.
	$$
	Adding the two bounds together completes the proof.
\end{proof}

\begin{extheorem} \label{exth: cen}
	Assume that the expected total variation distance between the learned dynamics model and the real dynamics model is bounded at each timestep under the expectation of $\pi_D^i$ and $\pi_{D}^{-i}$ by 
	$\epsilon_m=\max_t\mathbb{E}_{(s_{t},a_{t}^{i},a_{t}^{-i})\sim \pi_{D}^{i},\pi_{D}^{-i}}[D_{TV}(\mathcal{T}(\cdot|s_{t},a_{t}^{i},
	a_{t}^{-i})\|\hat{\mathcal{T}}(\cdot|s_{t},a_{t}^{i},a_{t}^{-i}))]$
	, and the distance of the policies are bounded as 
	$\epsilon_{\pi}^i=\max_{s}D_{TV}(\pi^i(\cdot|s)\|\pi_D^i(\cdot|s))$, $\epsilon_{\pi}^{j}=\max_{s}D_{TV}(\pi^{j}(\cdot|s)\|\pi_{D}^{j}(\cdot|s))$ 
	for $j \in \{-i\}$. Then the returns under the real dynamics and the learned dynamics model are bounded as:
	\begin{align*}
	\left|\eta_i[\pi^{i}, \pi^{-i}] - \hat{\eta}_i[\pi^{i}, \pi^{-i}]\right| 
	&\leq 2r_{\text{max}} \Big[\frac{\gamma(\epsilon_m+2\epsilon_{\pi}^i+2\sum_{j\in\{-i\}}\epsilon_{\pi}^{j})}{1-\gamma}+2\epsilon^{i}_{\pi}+2\sum_{j\in\{-i\}}\epsilon_{\pi}^{j} \Big] \\
	&=C(\epsilon_m, \epsilon_{\pi}^i, \epsilon_{\pi}^{-i})~,
	\end{align*}
	where $\epsilon_{\pi}^{-i}=\sum_{j\in\{-i\}}\epsilon_{\pi}^{j}$ is the distribution shift caused by the opponent agent policies.
\end{extheorem}

\begin{proof}
	Following the similar procedures of proving Theorem \ref{th: cen_bound} while using Extended Lemmas (\ref{exle: bl1}-\ref{exle: bl4}) completes the proof.
\end{proof}


The discrepancy bound of the returns in \Cref{exth: cen} is denoted as $C(\epsilon_m, \epsilon_{\pi}^i, \epsilon_{\pi}^{-i})$, in which the generalizatoin error of the dynamics model $\epsilon_m$ is the only controllable variable, i.e., can be reduced by improving the performance of the dynamics model\footnote{We are aware of the possibility of controlling the policy distribution shift terms $\epsilon_{\pi}^i, \epsilon_{\pi}^{-i}$ by performing trust region based methods, like SLBO \cite{luo2019algorithmic}. However, trust region based methods constrain the distribution shift between the current policy $\pi$ and the \emph{last round} policy $\pi_D$. By contrast, in our discussion, $\pi_D$ represents the policies in all previous rounds.}. 
However, in a centralized model-based MARL method, amounts of interactions between the agents are required, which leads to low sample efficiency.


\section{Upper Bounds for Model-based Multi-agent Policy Optimization with Adaptive Opponent-wise Rollouts}
\label{app: bounds}

In this section, we prove the upper bounds presented in the main paper for decentralized model-based MARL, from the perspective of the ego agent, i.e., agent $i$, where opponent modeling is required to infer the actions of opponent agents. Specifically, We first prove the upper bounds for the situations where the ego agent plays with one opponent, and then prove the upper bounds for the situations where the ego agent plays with multiple opponents.

\subsection{Single Opponent Case}

\begin{theorem}
	Assume that the expected total variation distance between the learned dynamics model $\hat{\mathcal{T}}$ and the real dynamics model $\mathcal{T}$ is bounded at each timestep under the expectation of $\pi_D^i$ and $\pi_D^{-i}$ by $\epsilon_m=\max_t\mathbb{E}_{(s_{t},a_{t}^{i},a_{t}^{-i})\sim\pi_{D}^{i},\pi_{D}^{-i}}[D_{TV}(\mathcal{T}(\cdot|s_{t},a_{t}^i,a_{t}^{-i})\|$
	$\hat{\mathcal{T}}(\cdot|s_{t},a_{t}^i,a_{t}^{-i}))]$, the distribution shift of the policies are bounded as $\epsilon_{\pi}^i=\max_{s}D_{TV}(\pi^i(\cdot|s)\|\pi_D^i(\cdot|s))$, $\epsilon_{\pi}^{-i}=\max_{s}D_{TV}(\pi^{-i}(\cdot|s)\|\pi_D^{-i}(\cdot|s))$,
	the generalization error of the opponent model is bounded as $\epsilon_{\hat{\pi}}=\max_{s}D_{TV}(\pi^{-i}(\cdot|s)\|$
	$\hat{\pi}^{-i}(\cdot|s))$.
	Then the discrepancy of returns under the real dynamics and the learned dynamics with opponent models is bounded as
	\begin{align*}
	\left| \eta_i[\pi^{i}, \pi^{-i}] - \hat{\eta}_i[\pi^{i}, \hat{\pi}^{-i}] \right|
	\leq 2r_{\text{max}}\Big[\frac{\gamma(\epsilon_m+2\epsilon_{\pi}^i+2\epsilon_{\pi}^{-i}+\epsilon_{\hat{\pi}}^{-i})}{1-\gamma}+2\epsilon^{i}_{\pi}+2\epsilon_{\pi}^{-i}+\epsilon_{\hat{\pi}}^{-i}\Big].
	\end{align*}
\end{theorem}

\begin{proof}
	By adding and subtracting a same term $\hat{\eta}_i[\pi^{i},\pi^{-i}]$, we have
	$$
	\left|\eta_i[\pi^{i}, \pi^{-i}] - \hat{\eta}_i[\pi^{i}, \hat{\pi}^{-i}]\right| \leq \underbrace{|\eta_i[\pi^{i}, \pi^{-i}] - \hat{\eta}_i[\pi^{i},\pi^{-i}]|}_{L_1} + \underbrace{|\hat{\eta}_i[\pi^{i},\pi^{-i}] - \hat{\eta}_i[\pi^{i}, \hat{\pi}^{-i}]|}_{L_2}~.
	$$
	We can bound both $L_1$ and $L_2$ using \Cref{th: cen_bound} and \Cref{le: bl3}.
	
	For $L_1$, we apply \Cref{th: cen_bound} and obtain
	$$
	L_1 \leq 2r_{\text{max}} \Big[\frac{\gamma(\epsilon_m+2\epsilon_{\pi}^i+2\epsilon_{\pi}^{-i})}{1-\gamma}+2\epsilon^{i}_{\pi}+2\epsilon^{-i}_{\pi}\Big]~.
	$$
	
	
	For $L_2$, we apply \Cref{le: bl3} using $M=0$, $P=0$, $O=\epsilon_{\hat{\pi}}^{-i}$, and obtain
	$$
	L_2 \leq 2r_{\text{max}}\Big[\frac{\gamma\epsilon_{\hat{\pi}}^{-i}}{1-\gamma}+\epsilon_{\hat{\pi}}^{-i}\Big]~.
	$$
	Adding together the two bounds completes the proof. 
\end{proof}



\begin{theorem}
	Assume that the expected total variation distance between the learned dynamics model and the real dynamics transition is bounded at each timestep under the expectation of $\pi^i$ and $\pi^{-i}$ by $\epsilon_m^{\prime}=\max_t\mathbb{E}_{(s_{t},a_{t}^{i},a_{t}^{-i})\sim\pi^i,\pi^{-i}}[D_{TV}(\mathcal{T}(\cdot|s_{t},a_{t}^i,a_{t}^{-i})\|$ $\hat{\mathcal{T}}(\cdot|s_{t},a_{t}^i,a_{t}^{-i}))]$, the distribution shift of the policies are bounded as $\epsilon_{\pi}^i=\max_{s}D_{TV}(\pi^i(\cdot|s)\|\pi_D^i(\cdot|s))$, $\epsilon_{\pi}^{-i}=\max_{s}D_{TV}$
	$(\pi^{-i}(\cdot|s)\|\pi_D^{-i}(\cdot|s))$,
	and the generalization error of the opponent model is bounded as $\epsilon_{\hat{\pi}}^{-i}=\max_{s}D_{TV}(\pi^{-i}(\cdot|s)\|\pi_{\phi}^{-i}(\cdot|s))$. 
	Then the upper bound of the return discrepancy is
	\begin{align*}
	& \Big|\eta_i[\pi^{i}, \pi^{-i}] - \eta_i^{\text{branch}}[(\pi^{i}_{D},\pi^{i}), (\pi^{-i}_{D},\hat{\pi}^{-i})]\Big| \\
	& \leq 2r_{\text{max}}\Big[k\epsilon_m^{\prime}+\gamma^{k+1}(\epsilon_{\pi}^{i}+\epsilon_{\pi}^{-i})+\frac{\gamma^{k+1}(\epsilon_{\pi}^{i}+\epsilon_{\pi}^{-i})}{1-\gamma}+(k+1)\epsilon_{\hat{\pi}}^{-i}\Big]~.
	\end{align*}
\end{theorem}

\begin{proof}
	In order to introduce the model generalization error, we select the reference term as the return of a branched rollout $\eta_i^{\text{branch}}[(\pi_D^i, \pi^i),$
	$(\pi_D^{-i}, \pi^{-i})]$, which executes $\pi_D^i$ and $\pi_D^{-i}$ in the real environment until the branch point and executes $\pi^i$ and $\pi^{-i}$ for $k$ steps with the dynamics model after the branch point. As such, we have
	\begin{align*}
	& \big|\eta_i[\pi^{i}, \pi^{-i}] - \eta_i^{\text{branch}}[(\pi^{i}_{D},\pi^{i}), (\pi^{-i}_{D},\hat{\pi}^{-i})]\big| \\
	& \leq ~\underbrace{\big|\eta_i[\pi^{i}, \pi^{-i}] - \eta_i^{\text{branch}}[(\pi_D^i, \pi^i), (\pi_D^{-i}, \pi^{-i})]\big|}_{L_1} 
	+\underbrace{\big|\eta_i^{\text{branch}}[(\pi_D^i, \pi^i), (\pi_D^{-i}, \pi^{-i})] - \eta_i^{\text{branch}}[(\pi^{i}_{D},\pi^{i}), (\pi^{-i}_{D},\hat{\pi}^{-i})]\big|}_{L_2}~,
	\end{align*}
	where we can bound both $L_1$ and $L_2$ using \Cref{le: bl4}.
	
	$L_1$ suffers from the policy errors before the branch point and the model generalization error after the branch point. By applying \Cref{le: bl4} with the bounds  $P^{\text{pre}}=\epsilon^i_{\pi}$, $O^{\text{pre}}=\epsilon_{\pi}^{-i}$, $M^{\text{post}}=\epsilon_m^{\prime}$ and setting other errors to $0$, we have
	$$
	L_1 \leq 2r_{\text{max}}\Big[k\epsilon_m^{\prime}+\gamma^{k+1}(\epsilon_{\pi}^i+\epsilon_{\pi}^{-i})+\frac{\gamma^{k+1}(\epsilon^i_{\pi}+\epsilon^{-i}_{\pi})}{1-\gamma}\Big]~.
	$$
	
	$L_2$ only suffers from the opponent model errors after the branch point, and we can apply \Cref{le: bl4} with the bounds  $O^{\text{post}}=\epsilon_{\hat{\pi}}^{-i}$ and set other errors to $0$, which leads to
	$$
	L_2 \leq 2r_{\text{max}}[(k+1)\epsilon_{\hat{\pi}}^{-i}]~.
	$$
	Adding the two bounds completes the proof.
\end{proof}

\subsection{Multiple Opponents Case}\label{sec:multi-opponent-bound}

Now we extend the theorems above to the situations where there are $n$ agents. From the perspective of the ego agent, i.e., agent $i$, we assume that the opponent agents $\{-i\}$ are making decisions independently and agent $i$ takes actions depending on the opponent models of other agents $\{-i\}$, which factorizes the joint distribution as
$$
p(s,a^{1},\ldots,a^{i},\ldots,a^{n}) = p(s)\pi^{i}(a^{i}|s,a^{1},\ldots,a^{i-1},a^{i+1},\ldots,a^{n})\prod_{j \in \{-i\}}\pi^{j}(a^{j}|s)~.
$$


\begin{extheorem} \label{exth: bound1}
	Assume that the expected total variation distance between the learned dynamics model and the real dynamics model is bounded at each timestep under the expectation of $\pi_D^i$ and $\pi_{D}^{-i}$ by 
	$\epsilon_m = \max_t\mathbb{E}_{(s_{t},a_{t}^{i},a_{t}^{-i})\sim(\pi_{D}^{i},\pi_{D}^{-i})}[D_{TV}($
	$\mathcal{T}(\cdot|s,a^{i},a^{-i})\|\hat{\mathcal{T}}(\cdot|s,a^{i},a^{-i}))]$,  the distribution shift of the policies are bounded as $\epsilon_{\pi}^{i} = \max_{s}D_{TV}(\pi^i(\cdot|s)\|\pi_D^i(\cdot|s))$, $\epsilon_{\pi}^{j} = \max_{s}D_{TV}(\pi^{j}(\cdot|s)\|\pi_D^{j}(\cdot|s))$ for $j \in \{-i\}$, 
	and the generalization errors of opponent models are bounded as  $\epsilon_{\hat{\pi}}^j = \max_{s}D_{TV}(\pi^{j}(\cdot|s)\|\hat{\pi}^{j}(\cdot|s))$ for $j \in \{-i\}$, then the return discrepancy is bounded as
	\begin{align*}
	\left| \eta_i[\pi^{i}, \pi^{-i}] - \hat{\eta}_i[\pi^{i}, \hat{\pi}^{-i}] \right| \leq 2r_{\text{max}}\Big[\frac{\gamma(\epsilon_m+2\epsilon_{\pi}^i+2\sum_{j \in \{-i\}}\epsilon_{\pi}^{j}+\sum_{j \in \{-i\}}\epsilon_{\hat{\pi}}^j)}{1-\gamma}+2\epsilon^{i}_{\pi}+2\sum_{j \in \{-i\}}\epsilon_{\pi}^{j}+\sum_{j \in \{-i\}}\epsilon_{\hat{\pi}}^j\Big].
	\end{align*}
\end{extheorem}

\begin{extheorem}
	\label{exth: bound2}
	Assume that the expected total variation distance between the learned dynamics model and the real dynamics transition is bounded at each timestep by $\epsilon_m^{\prime} = \max_t\mathbb{E}_{(s_{t},a_{t}^{i},a_{t}^{-i})\sim\pi^{i}, \pi^{-i}}[D_{TV}(\mathcal{T}(\cdot|s_{t},a_{t}^i,a_{t}^{-i})\|\hat{\mathcal{T}}(\cdot|s_{t},a_{t}^i,a_{t}^{-i}))]$, and the distribution shift of the policies are bounded as $\epsilon_{\pi}^{i} = \max_{s}D_{TV}(\pi^i(\cdot|s)\|\pi_D^i(\cdot|s))$, $\epsilon_{\pi}^{j} = \max_{s}D_{TV}(\pi^{j}(\cdot|s)\|\pi_D^{j}(\cdot|s))$ for $j \in \{-i\}$, 
	and the generalization errors of opponent models are bounded as $\epsilon_{\hat{\pi}}^j = \max_{s}D_{TV}(\pi^{j}(\cdot|s)\|\hat{\pi}^{j}(\cdot|s))$ for $j \in \{-i\}$, then the return discrepancy is bounded as 
	\begin{align*}
	& \Big| \eta_i[\pi^{i}, \pi^{-i}] -  \eta_i^{\text{branch}}[(\pi^{1}_{D},\hat{\pi}^1),\ldots,(\pi^{i}_{D},\pi^{i}),\ldots,(\pi^{n}_{D},\hat{\pi}^{n})] \Big| \nonumber \\
	& \leq 2r_{\text{max}}\Big[\underbrace{k\epsilon_m^{\prime}+ (k+1)\sum_{j \in \{-i\}}\epsilon_{\hat{\pi}}^j}_{\text{model generalization error}} + \underbrace{\gamma^{k+1}\big(\epsilon_{\pi}^{i}+\sum_{j \in \{-i\}}\epsilon_{\pi}^{j}\big)+\frac{\gamma^{k+1}(\epsilon_{\pi}^{i}+\sum_{j \in \{-i\}}\epsilon_{\pi}^{j})}{1-\gamma}}_{\text{distribution shift}} \Big] \nonumber \\
	& = C(\epsilon_m^{\prime}, \epsilon_{\pi}^{i}, \epsilon_{\pi}^{-i}, \epsilon_{\hat{\pi}}^{-i}, k)~.
	\end{align*}
\end{extheorem}

We can prove \Cref{exth: bound1} and \Cref{exth: bound1} following the similar procedures of proving Theorem B.1 and Theorem B.2 while using the extended lemmas and theorems.

\section{Useful Lemmas}

In this section, we prove the lemmas used throughout the paper. 
Specifically, we will first present 4 lemmas and then extend these lemmas into the situations where there are arbitrary number of variables.


	

\begin{lemma}\label{le: bl1} Assume that we have two joint distributions, i.e.,  $p_1(x,y,z)=p_1(x)p_1(y|x)p_1(z|x)$ and $p_2(x,y,z)=p_2(x)p_2(y|x)p_2(z|x)$\footnote{For conciseness, we slightly abuse the symbol $p_k$, and denote $D_{TV}(p_k(\cdot))$ w.r.t. variable $x$ as $D_{TV}(p_k(x))$ without conditional variables, for $k=1,2$. To avoid ambiguity, we denote $D_{TV}(p_k(\cdot|y))$ w.r.t. variable $z$ as $D_{TV}(p_k^{z}(\cdot|y))$ when the variable $z$ can not be deduced from the conditional variable $y$.}. We can bound the total variation distance between the two joint distributions as
	\begin{align*}
	& D_{TV}(p_1(x,y,z)\|p_2(x,y,z)) - D_{TV}(p_1(x)\|p_2(x))\\
	& \leq \max_xD_{TV}(p_1^{y}(\cdot|x)\|p_2^{y}(\cdot|x))+\max_{x}D_{TV}(p_1^{z}(\cdot|x)\|p_2^{z}(\cdot|x))~.
	\end{align*}
\end{lemma}

\begin{proof} Firstly, from Lemma B.1 of \cites{janner2019trust} we have the following inequality
	\begin{align}
	D_{T V} (p_{1}(x, y) \| p_{2}(x, y)) \leq D_{T V}(p_{1}(x) \| p_{2}(x))+\max_xD_{TV}(p_1^{y}(\cdot|x)\|p_2^{y}(\cdot|x))~,
	\label{eq: tvd_split}
	\end{align}
	where $p_1(x,y)=p_1(x)p_1(y|x)$ and $p_2(x,y)=p_2(x)p_2(y|x)$.
	
	Apply the \eq{eq: tvd_split} twice and we can get
	\begin{align*}
	& D_{TV}(p_1(x,y,z)\|p_2(x,y,z)) \\
	& \leq D_{TV}(p_1(x)\|p_2(x)) + \max_x D_{TV}(p_1^{y,z}(\cdot|x)\|p_2^{y,z}(\cdot|x)) \\
	& \leq D_{TV}(p_1(x)\|p_2(x)) + \max_x \big[D_{TV}(p_1^{y}(\cdot|x)\|p_2^{y}(\cdot|x))+\max_y D_{TV}(p_1^{z}(\cdot|x,y)\|p_2^{z}(\cdot|x,y))\big] \\
	& = D_{TV}(p_1(x)\|p_2(x)) + \max_x \big[D_{TV}(p_1^{y}(\cdot|x)\|p_2^{y}(\cdot|x))+ D_{TV}(p_1^{z}(\cdot|x)\|p_2^{z}(\cdot|x))\big] \\
	& \leq D_{TV}(p_1(x)\|p_2(x)) + \max_x D_{TV}(p_1^{y}(\cdot|x)\|p_2^{y}(\cdot|x)) +\max_x D_{TV}(p_1^{z}(\cdot|x)\|p_2^{z}(\cdot|x))~,
	\end{align*}
	which completes the proof.
\end{proof}

\begin{lemma}\label{le: bl2}
	
	Assume that the initial state distributions of two dynamics are the same $p_1(s_{t=0}) = p_2(s_{t=0})$. Suppose the distance between the two dynamics transitions is bounded as $M = \max_t \mathbb{E}_{(s_{t},a_{t}^{i},a_{t}^{-i}) \sim p_2}D_{TV}(p_1(\cdot|s_{t},a_{t}^i,a_{t}^{-i})\| p_2(\cdot|s_{t},a_{t}^i,a_{t}^{-i}))$, the distance between two  policies for agent $i$ is bounded as $P=\max_{s}D_{TV}(p_1(\cdot|s)\|p_2(\cdot|s))$, and the distance between two policies for agents $-i$ is bounded as  $O = \max_{s}D_{TV}(p_1(\cdot|s)\|p_2(\cdot|s))$. Then the distance in the state marginal at step $t$ is bounded as
	$$
	D_{TV}(p_1(s_t)\|p_2(s_t)) \leq t(M+P+O)~.
	$$
\end{lemma}

\begin{proof}
	
	Denote $\delta_t = D_{TV}(p_1(s_t)\|p_2(s_t))$, and we will prove the inequality in a recursive form.
	\begin{align*}
	\delta_t & = && D_{TV}(p_1(s_{t})\|p_2(s_{t})) \\
	& = && \frac{1}{2}\sum_{s_{t}}\big|p_1(s_{t})-p_2(s_{t})\big| \\
	& = && \frac{1}{2}\sum_{s_{t}}\Big|\sum_{s_{t-1},a_{t-1}^{i},a_{t-1}^{-i}} p_1(s_{t}|s_{t-1},a_{t-1}^{i},a_{t-1}^{-i})p_1(s_{t-1},a_{t-1}^{i},a_{t-1}^{-i})-p_2(s_{t}|s_{t-1},a_{t-1}^{i},a_{t-1}^{-i})p_2(s_{t-1},a_{t-1}^{i},a_{t-1}^{-i})\Big| \\
	& = && \frac{1}{2}\sum_{s_{t}}\Big|\sum_{s_{t-1},a_{t-1}^{i},a_{t-1}^{-i}} p_1(s_{t}|s_{t-1},a_{t-1}^{i},a_{t-1}^{-i})p_1(s_{t-1},a_{t-1}^{i},a_{t-1}^{-i})-p_1(s_{t}|s_{t-1},a_{t-1}^{i},a_{t-1}^{-i})p_2(s_{t-1},a_{t-1}^{i},a_{t-1}^{-i}) \\
	& && +p_1(s_{t}|s_{t-1},a_{t-1}^{i},a_{t-1}^{-i})p_2(s_{t-1},a_{t-1}^{i},a_{t-1}^{-i})-p_2(s_{t}|s_{t-1},a_{t-1}^{i},a_{t-1}^{-i})p_2(s_{t-1},a_{t-1}^{i},a_{t-1}^{-i})\Big| \\
	& \leq && \frac{1}{2}\sum_{s_{t}}\sum_{s_{t-1},a_{t-1}^{i},a_{t-1}^{-i}}\Big[\big|p_1(s_{t}|s_{t-1},a_{t-1}^{i},a_{t-1}^{-i})(p_1(s_{t-1},a_{t-1}^{i},a_{t-1}^{-i})-p_2(s_{t-1},a_{t-1}^{i},a_{t-1}^{-i}))\big| \\
	& &&+ \big|p_2(s_{t-1},a_{t-1}^{i},a_{t-1}^{-i})(p_1(s_{t}|s_{t-1},a_{t-1}^{i},a_{t-1}^{-i})-p_2(s_{t}|s_{t-1},a_{t-1}^{i},a_{t-1}^{-i}))\big|\Big] \\
	& = && \frac{1}{2}\sum_{s_{t-1},a_{t-1}^{i},a_{t-1}^{-i}}\big|p_1(s_{t-1}, a_{t-1}^{i}, a_{t-1}^{-i})-p_2(s_{t-1}, a_{t-1}^{i}, a_{t-1}^{-i})\big| \\
	& && + \mathbb{E}_{(s_{t-1},a_{t-1}^{i},a_{t-1}^{-i}) \sim p_2}\big[D_{TV}(p_1(\cdot|s_{t-1},a_{t-1}^i,a_{t-1}^{-i})\| p_2(\cdot|s_{t-1},a_{t-1}^i,a_{t-1}^{-i}))\big] \\
	& = && D_{TV}(p_1(s_{t-1},a_{t-1}^{i},a_{t-1}^{-i})\|p_2(s_{t-1},a_{t-1}^{i},a_{t-1}^{-i})) \\
	& && + \mathbb{E}_{(s_{t-1},a_{t-1}^{i},a_{t-1}^{-i}) \sim p_2}\big[D_{TV}(p_1(\cdot|s_{t-1},a_{t-1}^{i},a_{t-1}^{-i})\| p_2(\cdot|s_{t-1},a_{t-1}^{i},a_{t-1}^{-i}))\big] \\
	& \leq&& D_{TV}(p_1(s_{t-1})\|p_2(s_{t-1}))
	+ \max_{s_{t-1}}D_{TV}(p_1^{a^{-i}_{t-1}}(\cdot|s_{t-1})\|p_2^{a^{-i}_{t-1}}(\cdot|s_{t-1})) \tag{\text{\Cref{le: bl1}}} \\
    & && + \max_{s_{t-1}} D_{TV}(p_1^{a^{i}_{t-1}}(\cdot|s_{t-1})\|p_2^{a^{i}_{t-1}}(\cdot|s_{t-1}))
	+ M \\
	& = && \delta_{t-1} + M + P + O~.
	\end{align*}
	
	Note that $\delta_0 = 0$ since $p_1(s_{t=0}) = p_2(s_{t=0})$, so we have
	$$
	\delta_t \leq t(M+P+O)~,
	$$
	which completes the proof.
\end{proof}

In the following lemma, we prove a general discrepancy bound of returns using different policies under different dynamics and with different opponents from the perspective of the ego agent, i.e., agent $i$.

\begin{lemma}\label{le: bl3}
	Assume that the reward of agent $i$ is bounded as $r_{\max} = \max_{s,a^i,a^{-i}} r(s,a^i,a^{-i})$, the expected total variation distance between two dynamics distributions is bounded as $M = \max_t \mathbb{E}_{(s_{t-1},a_{t-1}^{i},a_{t-1}^{-i}) \sim p_2}[D_{TV}(p_1(\cdot|s_{t-1},a_{t-1}^i,a_{t-1}^{-i})\|$ $ p_2(\cdot|s_{t-1},a_{t-1}^i,a_{t-1}^{-i}))]$, and the total variation distance among the policies are $P = \max_{s}D_{TV}(p_1(\cdot|s)\|p_2(\cdot|s)) $, and $O = \max_{s}D_{TV}(p_1(\cdot|s)\|$
	$p_2(\cdot|s))$. Then the returns can be bounded as
	$$
	\big|\eta_i[\pi_1^{i}, \pi_1^{-i}]-\eta_i[\pi_2^{i}, \pi_2^{-i}]\big| \leq 2r_{\text{max}}\Big[\frac{\gamma(M+P+O)}{1-\gamma}+P+O \Big]~,
	$$
	where the terms $\eta_i[\pi_k^{i}, \pi_k^{-i}]$ for $k=1,2$ are the returns of running $\pi_k^{i}$ and $\pi_k^{-i}$ under dynamics transition $p_k(\cdot|s,a^{i},a^{-i})$.
\end{lemma}

\begin{proof}
	
	For policies $\pi^{i}$ and $\pi^{-i}$, the normalized occupancy measure in MARL is defined as
	\begin{align*}
	\rho(s,a^{i},a^{-i})
	&=(1-\gamma)\pi^{i}(a^{i}|s)\pi^{-i}(a^{-i}|s)\sum_{t=0}^{\infty}\gamma^{t}p(s_t=s|\pi^{i},\pi^{-i}) \\ 
	&=(1-\gamma)\sum_{t=0}^{\infty}\gamma^{t}p(s_t=s,a_t^{i}=a^{i},a_t^{-i}=a^{-i})~.
	\end{align*}
	Using this definition, we can equivalently express the MARL objective (expected return) as 
	$\eta_i[\pi^{i}, \pi^{-i}]=\sum_{s,a^{i},a^{-i}}\rho_{\pi^{i},\pi^{-i}}(s,a^{i},$
	$a^{-i})r(s,a^{i},a^{-i})$. 
    
    Then we have
	\begin{align*}
	& && \big|\eta_i[\pi_1^{i}, \pi_1^{-i}]-\eta_i[\pi_2^{i}, \pi_2^{-i}]\big| \\
	& = && \Big|\sum_{s,a^{i},a^{-i}}(\rho_1(s_{t}=s,a_{t}^{i}=a^{i},a_{t}^{-i}=a^{-i})-\rho_2(s_{t}=s,a_{t}^{i}=a^{i},a_{t}^{-i}=a^{-i}))r(s_{t}=s,a_{t}^{i}=a^{i},a_{t}^{-i}=a^{-i})\Big| \\
	& = && \Big|\sum_{s,a^{i},a^{-i}}(1-\gamma)\sum_t\gamma^t(p_1(s_{t}=s,a_{t}^{i}=a^{i},a_{t}^{-i}=a^{-i})-p_2(s_{t}=s,a_{t}^{i}=a^{i},a_{t}^{-i}=a^{-i})) \\
	& && r(s_{t}=s,a_{t}^{i}=a^{i},a_{t}^{-i}=a^{-i})\Big| \\
	& \leq && (1-\gamma)r_{\max}\sum_t\gamma^t\sum_{s,a^{i},a^{-i}}\big|p_1(s_{t}=s,a_{t}^{i}=a^{i},a_{t}^{-i}=a^{-i})-p_2(s_{t}=s,a_{t}^{i}=a^{i},a_{t}^{-i}=a^{-i})\big| \\
	& = && 2(1-\gamma)r_{\max}\sum_t\gamma^tD_{TV}(p_1(s_{t},a_{t}^{i},a_{t}^{-i})\|p_2(s_{t},a_{t}^{i},a_{t}^{-i})) \\
	& \leq && 2(1-\gamma)r_{\max}\sum_t\gamma^t(D_{TV}(p_1(s_{t})\|p_2(s_{t}))+P+O) \tag{{\text{\Cref{le: bl1}}}} \\
	& \leq && 2(1-\gamma)r_{\max}\sum_t\gamma^t(t(M+P+O)+P+O) \tag{\text{\Cref{le: bl2}}} \\
	& = && 2(1-\gamma)r_{\max}\Big[\frac{\gamma(M+P+O)}{(1-\gamma)^2}+\frac{P+O}{1-\gamma}\Big] \\
	& = && 2r_{\max}\Big[\frac{\gamma(M+P+O)}{1-\gamma}+P+O\Big]~.
	\end{align*}
\end{proof}

\begin{lemma}\label{le: bl4}
	Suppose we consider situations with branched rollouts of length $k$. Assume that before the branch point the expected distance between the dynamics distributions is bounded as $M^{\text{pre}} = \max_t \mathbb{E}_{(s_{t-1},a_{t-1}^{i},a_{t-1}^{-i}) \sim p_2}D_{TV}(p_1^{\text{pre}}(\cdot|s_{t-1},a_{t-1}^i,a_{t-1}^{-i})\|$ $p_2^{\text{pre}}(\cdot|s_{t-1},a_{t-1}^i,a_{t-1}^{-i}))$, and after the branch point the expected distance between dynamics distributions is bounded as $M^{\text{post}} = \max_t \mathbb{E}_{(s_{t-1},a_{t-1}^{i},a_{t-1}^{-i}) \sim p_2}D_{TV}(p_1^{\text{post}}(\cdot|s_{t-1},a_{t-1}^i,a_{t-1}^{-i})\|p_2^{\text{post}}(\cdot|s_{t-1},a_{t-1}^i,a_{t-1}^{-i}))$. Likewise, the policy distribution shifts are bounded as $P^{\text{pre}}$, $P^{\text{post}}$ and the generalization errors are bounded as $O^{\text{pre}}$, $O^{\text{post}}$. The discrepancy between the return in real environment $\eta_i[\pi_1^i, \pi_1^{-i}]$ and the return in dynamics model with branched rollouts $\eta_i^{\text{branch}}[(\pi_{2,\text{pre}}^{i}, \pi_{2,\text{post}}^{i}),(\pi_{2,\text{pre}}^{-i}, \hat{\pi}_{2,\text{post}}^{-i})]$ are bounded as
	\begin{align*}
	& \big|\eta_i[\pi_1^i, \pi_1^{-i}] - \eta_i^{\text{branch}}[(\pi_{2,\text{pre}}^{i}, \pi_{2,\text{post}}^{i}),(\pi_{2,\text{pre}}^{-i}, \hat{\pi}_{2,\text{post}}^{-i})]\big| \\
	& \leq 2r_{\text{max}} \Big[\frac{\gamma^{k+1}(M^{\text{pre}}+P^{\text{pre}}+O^{\text{pre}})}{1-\gamma}+\gamma^{k+1}(P^{\text{pre}}+O^{\text{pre}}) +k(M^{\text{post}}+P^{\text{post}}+O^{\text{post}})+P^{\text{post}}+O^{\text{post}}\Big]~,
	\end{align*}
	where $\eta_i^{\text{branch}}[(\pi_{2,\text{pre}}^{i}, \pi_{2,\text{post}}^{i}),(\pi_{2,\text{pre}}^{-i}, \pi_{2,\text{post}}^{-i})]$ denotes the expected return of running $\pi_{2,\text{pre}}^{i} $ and $\pi_{2,\text{pre}}^{-i} $ in real environment for several steps, and then running $\pi_{2,\text{post}}^{i} $ and $\hat{\pi}_{2,\text{post}}^{-i} $ on the dynamics model for $k$ steps.
\end{lemma}

\begin{proof}
	
	
	Using \Cref{le: bl1} and \Cref{le: bl2}, we can bound the state-action marginal error before and after the branch point respectively similar to the proof in MBPO \cite{janner2019trust}. 
	
	For $t \leq k$:
	\begin{equation*}
	\begin{aligned}
	D_{TV}(p_1(s_{t},a_{t}^i,a_{t}^{-i})\| p_2(s_{t},a_{t}^i,a_{t}^{-i}))
	& \leq t(M^{\text{post}}+P^{\text{post}}+O^{\text{post}})+P^{\text{post}}+O^{\text{post}} \\
	& \leq k(M^{\text{post}}+P^{\text{post}}+O^{\text{post}})+P^{\text{post}}+O^{\text{post}}~.
	\end{aligned}
	\end{equation*}
	
	For $t > k$:
	\begin{equation*}
	\begin{aligned}
	D_{TV}(p_1(s_{t},a_{t}^i,a_{t}^{-i})\| p_2(s_{t},a_{t}^i,a_{t}^{-i}))
	& \leq (t-k)(M^{\text{pre}}+P^{\text{pre}}+O^{\text{pre}})+P^{\text{pre}}+O^{\text{pre}} \\
	& \quad + k(M^{\text{post}}+P^{\text{post}}+O^{\text{post}})+P^{\text{post}}+O^{\text{post}}~.
	\end{aligned}
	\end{equation*}
	
    Then we can bound the difference in occupancy measures $\rho_1(s,a^{i},a^{-i})$ and $\rho_2(s,a^{i},a^{-i})$ as
	\begin{equation*}
	\begin{aligned}
	D_{TV}(\rho_1\|\rho_2) &\leq (1-\gamma)\sum_{t=0}^{\infty}\gamma^tD_{TV}(p_1(s_{t},a_{t}^i,a_{t}^{-i})\| p_2(s_{t},a_{t}^i,a_{t}^{-i})) \\
	& \leq (1-\gamma)\sum_{t=0}^{k}\gamma^t \big(k(M^{\text{post}}+P^{\text{post}}+O^{\text{post}})+P^{\text{post}}+O^{\text{post}}\big) \\
	& \quad+ (1-\gamma)\sum_{t=k+1}^{\infty}\gamma^t\big((t-k)(M^{\text{pre}}+P^{\text{pre}}+O^{\text{pre}})+P^{\text{pre}}+O^{\text{pre}} \\
	& \quad+ k(M^{\text{post}}+P^{\text{post}}+O^{\text{post}})+P^{\text{post}}+O^{\text{post}}\big) \\
	& \leq \gamma^{k+1}(P^{\text{pre}}+O^{\text{pre}})+\frac{\gamma^{k+1}(M^{\text{pre}}+P^{\text{pre}}+O^{\text{pre}})}{1-\gamma} \\
	& \quad+ k(M^{\text{post}}+P^{\text{post}}+O^{\text{post}}) + P^{\text{post}}+O^{\text{post}}~.
	\end{aligned}
	\end{equation*}
	
	As such, we can drive the result of this lemma via
	\begin{equation*}
	\begin{aligned}
	& \big|\eta_i[\pi_1^i, \pi_1^{-i}] - \eta_i^{\text{branch}}[(\pi_{2,\text{pre}}^{i}, \pi_{2,\text{post}}^{i}),(\pi_{2,\text{pre}}^{-i}, \hat{\pi}_{2,\text{post}}^{-i})]\big|\\
	& = \Big|\sum_{s,a^{i},a^{-i}}(\rho_1(s,a^{i},a^{-i})-\rho_2(s,a^{i},a^{-i}))r(s,a^{i},a^{-i})\Big| \\
	& \leq 2r_{\text{max}}D_{TV}(\rho_1\|\rho_2) \\
	& \leq 2r_{\text{max}}\Big[\frac{\gamma^{k+1}(M^{\text{pre}}+P^{\text{pre}}+O^{\text{pre}})}{1-\gamma}+\gamma^{k+1}(P^{\text{pre}}+O^{\text{pre}}) +k(M^{\text{post}}+P^{\text{post}}+O^{\text{post}})+P^{\text{post}}+O^{\text{post}}\Big]~.
	\end{aligned}
	\end{equation*}
\end{proof}

\begin{exlemma} \label{exle: bl1}
	
	Assume that we have two joint distributions $p_1(x, y_1,\ldots,y_{n-1}, z)=p_1(x)p_1(y_1|x)\cdots $ $p_1(y_{n-1}|x)p_1(z|x)$ and $p_2(x, y_1,\ldots,y_{n-1},z)=p_2(x)p_2(y_1|x)\cdots p_2(y_{n-1}|x) p_2(z|x)$, where $y_j$ and $y_k$ are conditionally independent given $x$, for $j,k \in \{1,\ldots , n-1\}$. We can bound the total variation distance between the joint distributions as
	\begin{align*}
	& D_{TV}(p_1(x, y_1,\ldots,y_{n-1},z)\|p_2(x, y_1,\ldots,y_{n-1},z)) \\
	& \leq D_{TV}(p_1(x)\|p_2(x))  + \sum_{j=1}^{n-1}\max_{x}D_{TV}(p_1^{y_j}(\cdot|x)\|p_2^{y_j}(\cdot|x))  + \max_{x}D_{TV}(p_1^{z}(\cdot|x)\|p_2^{z}(\cdot|x))~.
	\end{align*}
\end{exlemma}

\begin{proof}
	Because $y_i$ and $y_j$ are conditionally independent given $x$, for $i,j\in \{1,\ldots, n-1\}$, we have
	\begin{equation*}
    	D_{TV}(p_1(y_1,\ldots,y_{n-1}|x)\|p_2(y_1,\ldots,y_{n-1}|x))=\sum_{j=1}^{n-1}\max_{x}D_{TV}(p_1^{y_j}(\cdot|x)\|p_2^{y_j}(\cdot|x))~.
	\end{equation*} 
	
	Using Lemma \ref{le: bl1}, we have
	\begin{equation*}
	\begin{aligned}
	& D_{TV}(p_1(x, y_1,\ldots,y_{n-1},z)\|p_2(x, y_1,\ldots,y_{n-1},z)) \\
	& \leq D_{TV}(p_1(x)\|p_2(x)) + \max_{x}D_{TV}(p_1^{y_1,\ldots,y_{n-1}}(\cdot|x)\|p_2^{y_1,\ldots,y_{n-1}}(\cdot|x)) + \max_{x}D_{TV}(p_1^{z}(\cdot|x)\|p_2^{z}(\cdot|x)) \\
	& = D_{TV}(p_1(x)\|p_2(x)) + \sum_{j=1}^{n-1}\max_{x}D_{TV}(p_1^{y_j}(\cdot|x)\|p_2^{y_j}(\cdot|x)) + \max_{x}D_{TV}(p_1^{z}(\cdot|x)\|p_2^{z}(\cdot|x))~.
	\end{aligned}
	\end{equation*}
\end{proof}

Assume an $n$-agent stochastic game, we extend the \Cref{le: bl2}, \Cref{le: bl3}, and \Cref{le: bl4} by using \Cref{exle: bl1} in the proofs. Specifically, let $x=s$, $z=a^i$ and $\{y_1,\ldots,y_{n-1} \} = a^{-i}$, then we have the following conclusions.

\begin{exlemma}\label{exle: bl2}
	
	Assume that in the $n$-agent stochastic game, the initial state distributions of the two dynamics are the same, i.e., $p_1(s_{t=0}) = p_2(s_{t=0})$. 
	Let the dynamics distance be bounded as
	$M=\max_t\mathbb{E}_{(s_{t-1},a_{t-1}^{i},a_{t-1}^{-i})\sim p_2^{t-1}}[$ $D_{TV}(p_1(\cdot|s_{t-1},a_{t-1}^{i} a_{t-1}^{-i})\|p_2(\cdot|s_{t-1},a_{t-1}^{i}, a_{t-1}^{-i}))]$,
	the ego-agent policy distribution shift be bounded as $P=\max_{s}D_{TV}($
	$p_1(\cdot|s)\|p_2(\cdot|s))$,
	and the opponent model generalization errors be bounded as
	$O_j=\max_{s}D_{TV}(p_1(\cdot|s)\|p_2(\cdot|s))$ 
	for $j\in \{-i\}$, then the total variation distance in the state marginal is bounded as
	$$
	D_{TV}(p_1(s_t)\|p_2(s_t)) \leq t \Big(M+P+\sum_{j\in\{-i\}}O_j \Big)~.
	$$
\end{exlemma}



\begin{exlemma} \label{exle: bl3}
	Assume that in the $n$-agent stochastic game, the scale of the immediate reward is bounded as $r_{\max} = \max_{s,a^i,a^{-i}} r(s,a^i,a^{-i})$, 
	the dynamics distance is bounded as
	$M = \max_t\mathbb{E}_{(s_{t-1},a_{t-1}^{i},a_{t-1}^{-i})\sim p_2}[D_{TV}(p_1(\cdot|s_{t-1},a_{t-1}^{i}, a_{t-1}^{-i})\|$
	$p_2(\cdot|s_{t-1},a_{t-1}^{i},a_{t-1}^{-i}))]$,
	the ego-agent policy distribution shift is bounded as $P = \max_{s}D_{TV}(p_1(\cdot|s)\|p_2(\cdot|s))$,
	and the opponent model generalization errors are bounded as
	$O_j = \max_{s}D_{TV}(p_1(\cdot|s)\|p_2(\cdot|s))$ 
	for $j\in \{-i\}$. 
	Then the returns discrepancy can be bounded as
	$$
	\big|\eta_i[\pi^{i}_1, \pi^{-i}_{1}]-\eta_i[\pi^{i}_2, \pi^{-i}_2] \big| \leq 2r_{\text{max}} \Big[\frac{\gamma(M+P+\sum_{j\in\{-i\}}O_j)}{1-\gamma}+P+\sum_{j\in\{-i\}}O_j \Big]~.
	$$
\end{exlemma}

\begin{exlemma} \label{exle: bl4}
	
	Suppose we consider situations with branched rollouts of length $k$ in the $n$-agent stochastic game.
	Assume that before the branch point the expected distance between the dynamics is bounded as 
	$$
	\max_t\mathbb{E}_{(s_{t-1},a_{t-1}^{i},a_{t-1}^{-i})\sim p_2}[D_{TV}(p_1^{\text{pre}}(\cdot|s_{t-1},a_{t-1}^{i}, a_{t-1}^{-i})\|p_2^{\text{pre}}(\cdot|s_{t-1},a_{t-1}^{i}, a_{t-1}^{-i}))]\leq M^{\text{pre}}~,
	$$
	and after the branch point the expected distance between dynamics distributions is bounded as
	$$
	\max_t\mathbb{E}_{(s_{t-1},a_{t-1}^{i},a_{t-1}^{-i})\sim p_2}[D_{TV}(p_1^{\text{post}}(\cdot|s_{t-1},a_{t-1}^{i}, a_{t-1}^{-i})\|p_2^{\text{post}}(\cdot|s_{t-1},a_{t-1}^{i}, a_{t-1}^{-i}))]\leq M^{\text{post}}~.
	$$
	Likewise, the ego-agent policy distribution shifts are bounded as $P^{\text{pre}}$, $P^{\text{post}}$, $O^{\text{pre}}_j$ and the generalization errors of opponent models are bounded by $O^{\text{post}}_j$, for $j\in\{-i\}$. The discrepancy between the return in real environment $\eta_i[\pi_1^i, \pi_1^{-i}]$\ and the return in the dynamics model with branched rollouts $\eta_i^{\text{branch}}[(\pi_{2}^{i,\text{pre}}, \pi_{2}^{i,\text{post}}),(\pi_{2}^{-i,\text{pre}}, \hat{\pi}_{2}^{-i,\text{post}})]$ are bounded as
	\begin{equation*}
	\begin{aligned}
	& && \big|\eta_i[\pi^{i}_1, \pi^{-i}_{1}] -
	\eta_i^{\text{branch}}[(\pi_{2}^{i,\text{pre}}, \pi_{2}^{i,\text{post}}),(\pi_{2}^{-i,\text{pre}}, \hat{\pi}_{2}^{-i,\text{post}})]\big| \\
	& \leq && 2r_{\text{max}}\Big[\frac{\gamma^{k+1}(M^{\text{pre}}+P^{\text{pre}}+\sum_{j\in\{-i\}}O^{\text{pre}})}{1-\gamma}+\gamma^{k+1}(P^{\text{pre}}+\sum_{j\in\{-i\}}O^{\text{pre}}) \\
	& &&+k(M^{\text{post}}+P^{\text{post}}+\sum_{j\in\{-i\}}O^{\text{post}})+P^{\text{post}}+\sum_{j\in\{-i\}}O^{\text{post}}\Big]~.
	\end{aligned}
	\end{equation*}
\end{exlemma}


\section{Proof of Convergence}
\label{app: converge}

\begin{theorem}
	\label{th: MABRPO_convergence}
	Let $\mathcal{P}$ be the process operator of multi-agent soft Q-learning algorithm and $\boldsymbol{Q}^{\text{soft}}_{*}$ be its optimal value. Let $\{\mathcal{P}_t\}$ be the process operator of our algorithm at iteration $t$, which consists of a dynamics model, i.e., the transition model $\mathcal{T}_t(\cdot|s, \boldsymbol{a})$ and the reward model $\boldsymbol{r}_t(s, \boldsymbol{}{a})$\footnote{Take state-action pairs of all agents as input and the observations as well as rewards of all agents as output.} and opponent models, denoted as $\boldsymbol{\phi}_t(\cdot|s)$, for simplicity\footnote{Take the observations of all opponents as input and the actions of all opponents as output.}. Define $\boldsymbol{Q}_{t+1}=\mathcal{P}_t(\boldsymbol{Q}_{t})$.
	
	Define the operators $\oplus_t^{\boldsymbol{\phi}_t}$ and $\oplus^{\boldsymbol{\phi}_t}$ as following: 
	\begin{align*}
	\oplus_t^{\boldsymbol{\phi}_t}\boldsymbol{f}(\cdot) &= \mathbb{E}_{s^{\prime}\sim \mathcal{T}_t}\Big[\alpha\log\int_{\boldsymbol{A}^{-i}}\boldsymbol{\phi}_t(\boldsymbol{a}^{-i}|s^{\prime})\int_{\boldsymbol{A}^i}\exp(\boldsymbol{f}(\cdot)/\alpha)d\boldsymbol{a}^{i}d\boldsymbol{a}^{-i}\Big]~;\\
	\oplus^{\boldsymbol{\phi}_t}\boldsymbol{f}(\cdot) &= \mathbb{E}_{s^{\prime}\sim \mathcal{T}_s}\Big[\alpha\log\int_{\boldsymbol{A}^{-i}}\boldsymbol{\phi}_t(\boldsymbol{a}^{-i}|s^{\prime})\int_{\boldsymbol{A}^i}\exp(\boldsymbol{f}(\cdot)/\alpha)d\boldsymbol{a}^{i}d\boldsymbol{a}^{-i}\Big]~,
	\end{align*}
	where $\mathcal{T}_s(\cdot)$ means the real transition of the environment.
	
	The operator $\mathcal{P}$ can be expressed as
	$$
	\mathcal{P}(\boldsymbol{Q}_t, \boldsymbol{Q})=
	\begin{cases}
	\boldsymbol{r}_t(s, \boldsymbol{a})+\beta\gamma\oplus^{\boldsymbol{\phi}_t}\boldsymbol{Q}+(1-\beta)\gamma\oplus_t^{\boldsymbol{\phi}_t}\boldsymbol{Q} &  (s, \boldsymbol{a})\in \tau_t \\
	\boldsymbol{Q}_t &  \text{otherwise}
	\end{cases}~,
	$$
	where $\tau_t$ is the trajectory sample at iteration $t$.
	
	Then $\boldsymbol{Q}_t$ values computed using $\mathcal{P}_t$ will converge to $\boldsymbol{Q}^{\text{soft}}_*$ if the conditions below are satisfied:
	\begin{enumerate}[leftmargin=30px]
		\item $\oplus^{\boldsymbol{\phi}}_{t} \rightarrow \oplus^{\boldsymbol{\phi}}$, which means that
		$$
		\lim_{t \rightarrow \infty} \max _{(s, \boldsymbol{a}) \in \mathcal{S} \times A} | \oplus_{t}^{\boldsymbol{\phi}} \boldsymbol{f}(\cdot)-\oplus^{\boldsymbol{\phi}} \boldsymbol{f}(\cdot)|=\boldsymbol{0}~,
		$$
		for all functions $\boldsymbol{f}$.
		\item $\boldsymbol{r}_t$ converges to $\boldsymbol{r}$ for all $(s, \boldsymbol{a})$, $\boldsymbol{r}$ is the real reward function, and $\boldsymbol{r}_t$ is continuous for all $t$.
		$$
		\lim_{t\rightarrow\infty}\max_{s, \boldsymbol{a}}|\boldsymbol{r}_t-\boldsymbol{r}| = \boldsymbol{0}~,
		$$
		\item $\boldsymbol{\phi}_t$ converges to $\boldsymbol{\phi}$ for all $s$, $\boldsymbol{\phi}$ is the real joint opponent policy.
		$$
		\lim_{t\rightarrow\infty}\max_{s}|\boldsymbol{\phi}_t-\boldsymbol{\phi}| = \boldsymbol{0}~,
		$$
		\item Discount factor $\gamma \in [0,1)$.
		\item Each state-action pair is visited infinitely often.
	\end{enumerate}
\end{theorem}

\begin{proof}
	
	We will use Lemma D.5 to prove this theorem.
	Since $(s, \boldsymbol{a})$ are always in $\tau_t$, we use $\mathcal{P}(\boldsymbol{Q})=
	\boldsymbol{r}_t(s, \boldsymbol{a})+\beta\gamma\oplus^{\boldsymbol{\phi}_t}\boldsymbol{Q}+(1-\beta)\gamma\oplus_t^{\boldsymbol{\phi}_t}\boldsymbol{Q}
	$ in the proof below.
	
	First, we prove that $\mathcal{P}_t$ approximates $\mathcal{P}$ at arbitrary Q value. Let $\boldsymbol{D}_t = |\mathcal{P}_t(\boldsymbol{Q})-\mathcal{P}(\boldsymbol{Q})|$, we show that $\boldsymbol{D}_t$ converges to $\boldsymbol{0}$ when $t\rightarrow\infty$.
	Define $\boldsymbol{b}$ as the opponent action matrix, i.e., predicted opponent actions for each agent.
	\begin{align*}
	\boldsymbol{D}_t & = |\mathcal{P}_t(\boldsymbol{Q})-\mathcal{P}(\boldsymbol{Q})| \\
	& =\Big|\boldsymbol{r}_t(s, [\boldsymbol{a}^{i}\; \boldsymbol{b}]) + \beta\gamma\oplus^{\boldsymbol{\phi}_t}\boldsymbol{Q}+(1-\beta)\gamma\oplus_t^{\boldsymbol{\phi}_t}\boldsymbol{Q} -\boldsymbol{r}_t(s, \boldsymbol{a}) - \mathbb{E}_{s^{\prime}\sim \mathcal{T}_s}\Big[\alpha\log\int_{\boldsymbol{A}}\exp(\boldsymbol{Q}/\alpha)d\boldsymbol{a}\Big]\Big| \\
	& \leq |\boldsymbol{r}_t(s, [\boldsymbol{a}^{i}\; \boldsymbol{b}])-\boldsymbol{r}_t(s, \boldsymbol{a})| \\
	& \quad+\beta\gamma \Big|\mathbb{E}_{s^{\prime}\sim \mathcal{T}_s}\Big[\alpha\log\int_{\boldsymbol{A}^{-i}}\boldsymbol{\phi}_t(\boldsymbol{a}^{-i}|s^{\prime})\int_{\boldsymbol{A}^i}\exp(\boldsymbol{Q}/\alpha)d\boldsymbol{a}^{i}d\boldsymbol{a}^{-i}-\alpha\log\int_{\boldsymbol{A}}\exp(\boldsymbol{Q}/\alpha)d\boldsymbol{a}\Big]\Big| \\
	& \quad+(1-\beta)\gamma \Big|\mathbb{E}_{s^{\prime}\sim \mathcal{T}_t}\Big[\alpha\log\int_{\boldsymbol{A}^{-i}}\boldsymbol{\phi}_t(\boldsymbol{a}^{-i}|s^{\prime})\int_{\boldsymbol{A}^i}\exp(\boldsymbol{Q}/\alpha)d\boldsymbol{a}^{i}d\boldsymbol{a}^{-i}\Big] \\
	& \quad\quad\quad\quad\quad\quad-\mathbb{E}_{s^{\prime}\sim \mathcal{T}_s}\Big[\alpha\log\int_{\boldsymbol{A}}\exp(\boldsymbol{Q}/\alpha)d\boldsymbol{a}\Big]\Big| \\
	& = \boldsymbol{D}_1 + \beta\gamma\boldsymbol{D}_2 + (1-\beta)\gamma\boldsymbol{D}_3~.
	\end{align*}

	For $\boldsymbol{D}_1$, $\boldsymbol{\phi}_t$ converges to $\boldsymbol{\phi}$, thus $[\boldsymbol{a}^{i}\;\boldsymbol{b}]$ converges to $\boldsymbol{a}$. Given $\boldsymbol{r}_t$ is continuous, $\boldsymbol{D}_1$ converges to $\boldsymbol{0}$ when $t\rightarrow \infty$.
	
	For $\boldsymbol{D}_2$, we show that with the fourth condition,  $\boldsymbol{D}_2$ will converges to $\boldsymbol{0}$.
	\begin{align*}
	\boldsymbol{D}_2 & = \Big|\mathbb{E}_{s^{\prime}\sim \mathcal{T}_s}\Big[\alpha\log\int_{\boldsymbol{A}^{-i}}\boldsymbol{\phi}_t(\boldsymbol{a}^{-i}|s^{\prime})\int_{\boldsymbol{A}^i}\exp(\boldsymbol{Q}/\alpha)d\boldsymbol{a}^{i}d\boldsymbol{a}^{-i}-\alpha\log\int_{\boldsymbol{A}}\exp(\boldsymbol{Q}/\alpha)d\boldsymbol{a}\Big]\Big| \\
	& = \Big|\mathbb{E}_{s^{\prime}\sim \mathcal{T}_s}\Big[\alpha\log\frac{\int_{\boldsymbol{A}^{-i}}\boldsymbol{\phi}_t(\boldsymbol{a}^{-i}|s^{\prime})\int_{\boldsymbol{A}^i}\exp(\boldsymbol{Q}/\alpha)d\boldsymbol{a}^{i}d\boldsymbol{a}^{-i}}{\int_{\boldsymbol{A}^{-i}}\int_{\boldsymbol{A}^i}\exp(\boldsymbol{Q}/\alpha)d\boldsymbol{a}^{i}d\boldsymbol{a}^{-i}}\Big]\Big| \\
	& = \Big|\mathbb{E}_{s^{\prime}\sim \mathcal{T}_s}\Big[\alpha\log\frac{\int_{\boldsymbol{A}^{-i}}(\boldsymbol{\phi}(\boldsymbol{a}^{-i}|s^{\prime})+\boldsymbol{\phi}_t(\boldsymbol{a}^{-i}|s^{\prime})-\boldsymbol{\phi}(\boldsymbol{a}^{-i}|s^{\prime}))\int_{\boldsymbol{A}^i}\exp(\boldsymbol{Q}/\alpha)d\boldsymbol{a}^{i}d\boldsymbol{a}^{-i}}{\int_{\boldsymbol{A}^{-i}}\int_{\boldsymbol{A}^i}\exp(\boldsymbol{Q}/\alpha)d\boldsymbol{a}^{i}d\boldsymbol{a}^{-i}}\Big]\Big| \\
	& = \Big|\mathbb{E}_{s^{\prime}\sim \mathcal{T}_s}\Big[\alpha\log\frac{\int_{\boldsymbol{A}^{-i}}\boldsymbol{\phi}(\boldsymbol{a}^{-i}|s^{\prime})\int_{\boldsymbol{A}^i}\exp(\boldsymbol{Q}/\alpha)d\boldsymbol{a}^{i}d\boldsymbol{a}^{-i}}{\int_{\boldsymbol{A}^{-i}}\int_{\boldsymbol{A}^i}\exp(\boldsymbol{Q}/\alpha)d\boldsymbol{a}^{i}d\boldsymbol{a}^{-i}} \\
	& \quad\quad\quad\quad\quad + \frac{\int_{\boldsymbol{A}^{-i}}(\boldsymbol{\phi}_t(\boldsymbol{a}^{-i}|s^{\prime})-\boldsymbol{\phi}(\boldsymbol{a}^{-i}|s^{\prime}))\int_{\boldsymbol{A}^i}\exp(\boldsymbol{Q}/\alpha)d\boldsymbol{a}^{i}d\boldsymbol{a}^{-i}}{\int_{\boldsymbol{A}^{-i}}\int_{\boldsymbol{A}^i}\exp(\boldsymbol{Q}/\alpha)d\boldsymbol{a}^{i}d\boldsymbol{a}^{-i}}\Big]\Big|\\
	& \leq \Big|\mathbb{E}_{s^{\prime}\sim \mathcal{T}_s}\Big[\alpha\log\frac{\int_{\boldsymbol{A}^{-i}}\int_{\boldsymbol{A}^i}\exp(\boldsymbol{Q}/\alpha)d\boldsymbol{a}^{i}d\boldsymbol{a}^{-i}}{\int_{\boldsymbol{A}^{-i}}\int_{\boldsymbol{A}^i}\exp(\boldsymbol{Q}/\alpha)d\boldsymbol{a}^{i}d\boldsymbol{a}^{-i}} \\
	& \quad\quad\quad\quad\quad + \frac{\int_{\boldsymbol{A}^{-i}}(\boldsymbol{\phi}_t(\boldsymbol{a}^{-i}|s^{\prime})-\boldsymbol{\phi}(\boldsymbol{a}^{-i}|s^{\prime}))\int_{\boldsymbol{A}^i}\exp(\boldsymbol{Q}/\alpha)d\boldsymbol{a}^{i}d\boldsymbol{a}^{-i}}{\int_{\boldsymbol{A}^{-i}}\int_{\boldsymbol{A}^i}\exp(\boldsymbol{Q}/\alpha)d\boldsymbol{a}^{i}d\boldsymbol{a}^{-i}}\Big]\Big|\\
	& \leq \mathbb{E}_{s^{\prime}\sim \mathcal{T}_s}\Big[\alpha\log\Big(1+\frac{\int_{\boldsymbol{A}^{-i}}\max_{s^{\prime}, \boldsymbol{a}^{-i}}|\boldsymbol{\phi}_t(\boldsymbol{a}^{-i}|s^{\prime})-\boldsymbol{\phi}(\boldsymbol{a}^{-i}|s^{\prime})|\int_{\boldsymbol{A}^i}\exp(\boldsymbol{Q}/\alpha)d\boldsymbol{a}^{i}d\boldsymbol{a}^{-i}}{\int_{\boldsymbol{A}^{-i}}\int_{\boldsymbol{A}^i}\exp(\boldsymbol{Q}/\alpha)d\boldsymbol{a}^{i}d\boldsymbol{a}^{-i}}\Big)\Big] \\
	\end{align*}

	Given $t\rightarrow\infty$, the upper bound of $\boldsymbol{D}_2$ will converge to $\mathbb{E}_{s^{\prime}\sim \mathcal{T}_s}[\alpha\log(\boldsymbol{1})]=\boldsymbol{0}$.
	
	For $\boldsymbol{D}_3$, we split it into two parts:
	\begin{align*}
	\boldsymbol{D}_3 & =
	\Big|\mathbb{E}_{s^{\prime}\sim \mathcal{T}_t}\Big[\alpha\log\int_{\boldsymbol{A}^{-i}}\boldsymbol{\phi}_t(\boldsymbol{a}^{-i}|s^{\prime})\int_{\boldsymbol{A}^i}\exp(\boldsymbol{Q}/\alpha)d\boldsymbol{a}^{i}d\boldsymbol{a}^{-i}\Big]-\mathbb{E}_{s^{\prime}\sim \mathcal{T}_s}[\alpha\log\int_{\boldsymbol{A}}\exp(\boldsymbol{Q}/\alpha)d\boldsymbol{a}]\Big| \\
	& \leq \Big|\mathbb{E}_{s^{\prime}\sim \mathcal{T}_s}\Big[\alpha\log\int_{\boldsymbol{A}^{-i}}\boldsymbol{\phi}_t(\boldsymbol{a}^{-i}|s^{\prime})\int_{\boldsymbol{A}^i}\exp(\boldsymbol{Q}/\alpha)d\boldsymbol{a}^{i}d\boldsymbol{a}^{-i}-\alpha\log\int_{\boldsymbol{A}}\exp(\boldsymbol{Q}/\alpha)d\boldsymbol{a}\Big]\Big| \\
	& \quad+\Big|\mathbb{E}_{s^{\prime}\sim \mathcal{T}_t}\Big[\alpha\log\int_{\boldsymbol{A}}\exp(\boldsymbol{Q}/\alpha)d\boldsymbol{a}-\mathbb{E}_{s^{\prime}\sim \mathcal{T}_s}\Big[\alpha\log\int_{\boldsymbol{A}}\exp(\boldsymbol{Q}/\alpha)d\boldsymbol{a}\Big]\Big]\Big| \\
	& = \boldsymbol{D}_2 + \boldsymbol{D}_4~.
	\end{align*}
	
	We have already shown that $\boldsymbol{D}_2$ converges to $\boldsymbol{0}$ when $t\rightarrow\infty$. From condition 1, we have $\lim_{t\rightarrow\infty}\max_{s, \boldsymbol{a}}\boldsymbol{D}_4=\boldsymbol{0}$.
	
	Now we prove that $\mathcal{P}_t$ approximates $\mathcal{P}$ with arbitrary Q value function.
	To apply Lemma D.5, we check the 4 conditions in Lemma D.5.
	From Lemma D.6, we get $\oplus_t^{\boldsymbol{\phi}}$ is a non-expansion operator, for all opponent model $\boldsymbol{\phi}$.
	Since $\oplus_t^{\boldsymbol{\phi}_t}$ is a non-expansion operator, we get $\|\mathcal{P}_t(\boldsymbol{Q}_t, \boldsymbol{Q_{*}})-\mathcal{P}_t(\boldsymbol{Q}_t, \boldsymbol{Q})\|\leq \gamma\|\boldsymbol{Q_{*}}-\boldsymbol{Q}\|$.
	Then select
	$$
	G_{t}(x)=\begin{cases}
	{0,} & {\text { if } x \in \tau_{t}} \\
	{1,} & {\text { otherwise }}
	\end{cases}~,
	$$
	and
	$$
	F_{t}(x)=\begin{cases}
	{\gamma,} & {\text { if } x \in \tau_{t}} \\
	{0,}      & {\text { otherwise }}
	\end{cases}~.
	$$
	
	With condition 4 and condition 5, all three conditions in Lemma D.5 are satisfied. As such, the proof is completed.
\end{proof}

\begin{theorem} \label{th: MASQL_convergence}
	In a finite state $n$-agent stochastic game, the Q value sequence ${\boldsymbol{Q}_0, \boldsymbol{Q}_1, \ldots, \boldsymbol{Q}_n}$, as computed by the update rule of Multi-agent Soft Q learning algorithm
	\begin{equation}
	\boldsymbol{Q}_{t+1} = (1-\alpha_t)\boldsymbol{Q}_t + \alpha_t(\boldsymbol{r}_t+\gamma \mathbb{E}_{s^{\prime}\sim \mathcal{T}_s}[\boldsymbol{V}_t^{\text{soft}}\left(s^{\prime}\right)])~,
	\label{eq: masql}
	\end{equation}
	will converge to the Nash Q-value $\boldsymbol{Q}_{*}=\left[Q_{*}^{1}, \ldots, Q_{*}^{N_o}\right]$ if the following conditions are satisfied:
	
	\begin{enumerate}[leftmargin=15px]
		\item $0 \leq \alpha_{t}(x) \leq 1, \sum_{t} \alpha_{t}(x)=\infty, \sum_{t} \alpha_{t}^{2}(x)<\infty$, where $\alpha_t$ is the learning rate.
		\item Every state-action pair $(s, \boldsymbol{a})$ is visited infinitely often.
		\item For each stage game, the Nash equilibrium is recognized either as the global optimum or as a saddle point. All the global optimum or saddle points share the same Nash value in each stage game\footnote{The details of this condition can be found in \cites{yang2018mean}.}.
	\end{enumerate}
\end{theorem}

\begin{proof}
	We prove this theorem using Lemma D.1. Substracting $\boldsymbol{Q}_{*}$ on both sides of \eq{eq: masql}, we obtain
	$$
	\begin{aligned}
	\boldsymbol{\Delta}_{t}(x) & =\boldsymbol{Q}_{t}(s, \boldsymbol{a})-\boldsymbol{Q}_{*}(s, \boldsymbol{a}) \\
	\boldsymbol{F}_t(x) & =\boldsymbol{r}_t+\gamma \mathbb{E}_{s^{\prime}\sim \mathcal{T}_s}[\boldsymbol{V}_{t}^{\text{soft}}(s^{\prime})]-\boldsymbol{Q}_{*}(s, \boldsymbol{a})~,
	\end{aligned}
	$$
	where $x \triangleq\left(s, \boldsymbol{a}\right)$, $\boldsymbol{V}_t^{\text{soft}}\left(s\right)=[V_{t}^{\text{soft},1}, \ldots, V_{t}^{\text{soft},n}]$ and $V_{t}^{\text{soft},i}(s)=\alpha\log\int\exp(\boldsymbol{Q}_t^i(s,\boldsymbol{a})/\alpha)d\boldsymbol{a}$.
	Let $\mathcal{F}_t$ denote the $\sigma$ field generated at time $t$ by the process. Since $\boldsymbol{Q}_t$ is derived from previous trajectory up to time $t$, both $\boldsymbol{\Delta}_t$ and $\boldsymbol{F}_t$ are $\mathcal{F}_t$-measurable. The condition 1 and condition 2 of Lemma D.1 are satisfied because of the Condition 1 of this theorem and the finite state space assumption.
	
	To apply Lemma D.1, we now show that $\boldsymbol{F}_t$ satisfies the Condition 3 and Condition 4 in Lemma D.1.
	\begin{align*}
	\boldsymbol{F}_t(x) & =\boldsymbol{r}_t+\gamma \mathbb{E}_{s^{\prime}\sim \mathcal{T}_s}\left[\boldsymbol{V}_t^{\text{soft}}\left(s^{\prime}\right)\right]-\boldsymbol{Q}_{*}(s, \boldsymbol{a}) \\
	& =\boldsymbol{r}_t+\gamma\boldsymbol{v}_t^{\text{Nash}}(s^{\prime}) -\boldsymbol{Q}_{*}(s, \boldsymbol{a})+\gamma(\mathbb{E}_{s^{\prime}\sim \mathcal{T}_s}\left[\boldsymbol{V}_t^{\text{soft}}\left(s^{\prime}\right)\right]-\boldsymbol{v}_t^{\text{Nash}}(s^{\prime})) \\
	& =\boldsymbol{F}_{t}^{\mathrm{Nash}}\left(s, \boldsymbol{a}\right)+\boldsymbol{c}_{t}\left(s, \boldsymbol{a}\right)~,
	\end{align*}
	where $\boldsymbol{v}_t^{\text{Nash}}(s)=[v_t^{\text{Nash},1},\ldots,v_t^{\text{Nash},n}]$ and $v_t^{i}(s)=\mathbb{E}[Q_t^{i}(s,\boldsymbol{a})]$.
	
	From Lemma D.2, $\boldsymbol{F}_{t}^{\mathrm{Nash}}$ forms a contraction mapping with norm $\|\cdot\|$ as the maximum norm on $\boldsymbol{a}$. Thus we have $\|\mathbb{E}[\boldsymbol{F}_{t}^{\mathrm{Nash}}\left(s, \boldsymbol{a}\right) | \mathcal{F}_{t}]\| \leq \gamma\left\|\boldsymbol{Q}_{t}-\boldsymbol{Q}_{*}\right\|=\gamma\left\|\boldsymbol{\Delta}_{t}\right\|$.
	Here we are left to prove that $\boldsymbol{c}_t$ converges to $0$ with $t$ increasing.
	
	From Lemma D.3 and Lemma D.4, the update operator $\mathcal{P}$ in multi-agent soft q learning algorithm is a contraction mapping and $\boldsymbol{Q}_t$ is monotonously improved. The policy, which is based on $\boldsymbol{Q}_t$ according to the algorithm, will converge. With the monotonously increased $\boldsymbol{Q}_t$, the policy will converge to a global optimum or a saddle point. Given that the global optimums and the saddle points share the same Nash value in each stage game, $\boldsymbol{V}_t$ will asymptotically converge to $\boldsymbol{v}^{\text{Nash}}$, the Condition 3 of Lemma D.1 is thus satisfied.
	
	For the Condition 4 of Lemma D.1, we firstly show that
	$$
	\begin{aligned}
	\mathbb{E}\left[\boldsymbol{F}_{t}\left(s, \boldsymbol{a}\right) | \mathcal{F}_{t}\right] & =\mathbb{E}_{s^{\prime}\sim \mathcal{T}_s}[\boldsymbol{r}_t(s, \boldsymbol{a}, s^{\prime})+\gamma\mathbb{E}_{s^{\prime}\sim \mathcal{T}_s}\left[\boldsymbol{V}_t^{\text{soft}}\left(s^{\prime}\right)\right]-\boldsymbol{Q}_{*}(s, \boldsymbol{a})
	] \\
	& =\mathbb{E}_{s^{\prime}\sim \mathcal{T}_s}[\boldsymbol{r}_t(s, \boldsymbol{a}, s^{\prime})]+\gamma\mathbb{E}_{s^{\prime}\sim \mathcal{T}_s}\left[\boldsymbol{V}_t^{\text{soft}}\left(s^{\prime}\right)\right]-\boldsymbol{Q}_{*}(s, \boldsymbol{a})~.
	\end{aligned}
	$$
	
	So the variance of $\boldsymbol{F}_t$ can be computed as
	$$
	\begin{aligned}
	\operatorname{var}\left[\boldsymbol{F}_{t}\left(s, \boldsymbol{a}\right) | \mathcal{F}_{t}\right] & =
	\mathbb{E}_{s^{\prime}}[((\boldsymbol{r}_t(s, \boldsymbol{a}, s^{\prime})+\gamma\mathbb{E}_{s^{\prime}\sim \mathcal{T}_s}\left[\boldsymbol{V}_t^{\text{soft}}\left(s^{\prime}\right)\right]-\boldsymbol{Q}_{*}(s, \boldsymbol{a})) \\
	& \quad-(\mathbb{E}_{s^{\prime}\sim \mathcal{T}_s}[\boldsymbol{r}_t(s, \boldsymbol{a}, s^{\prime})]+\gamma\mathbb{E}_{s^{\prime}\sim \mathcal{T}_s}\left[\boldsymbol{V}_t^{\text{soft}}\left(s^{\prime}\right)\right]-\boldsymbol{Q}_{*}(s, \boldsymbol{a})))^2] \\
	& = \boldsymbol{0}~.
	\end{aligned}
	$$
	
	Therefore the Condition 4 of Lemma D.1 is satisfied. Finally, we apply Lemma D.1 to show that $\boldsymbol{\Delta}_t$ converges to 0 w.p.1, i.e., $\boldsymbol{Q}_t$ converges to $\boldsymbol{Q}_{*}$ w.p.1.
\end{proof}

\begin{lemma} \label{le: cl1}
	A random process $\{\Delta_t\}$ defined as
	$$
	\Delta_{t+1}(x)=\left(1-\alpha_{t}(x)\right) \Delta_{t}(x)+\alpha_{t}(x) F_{t}(x)
	$$
	will converges to 0 w.p.1 if the following conditions are satisfied:
	\begin{enumerate}[leftmargin=30px]
		\item $0 \leq \alpha_{t}(x) \leq 1, \sum_{t} \alpha_{t}(x)=\infty, \sum_{t} \alpha_{t}^{2}(x)<\infty$.
		\item $x \in \mathcal{X}$, the set of possible states, and $|\mathcal{X}|<\infty$.
		\item  $\left\|\mathbb{E}\left[F_{t}(x) | \mathcal{F}_{t}\right]\right\|_{W} \leq \gamma\left\|\Delta_{t}\right\|_{W}+c_{t},$ where $\gamma \in[0,1)$ and $c_t$ converges to 0 w.p.1.
		\item $\operatorname{var}\left[F_{t}(x) | \mathcal{F}_{t}\right] \leq K(1+\left\|\Delta_{t}\right\|_{W}^{2})$ with constant $K>0$.
	\end{enumerate}

	$\mathcal{F}_t$ is an increasing sequence of $\sigma$-fields including the history of processes; $\alpha_t, \Delta_t, F_t \in \mathcal{F}_t$ and $\|\cdot\|_w$ is a weighted maximum norm \cite{bertsekas2012weighted}.
\end{lemma}

\begin{proof}
	
	See Proposition 1 in \cites{jaakkola1994convergence}.
\end{proof}

\begin{lemma} \label{le: cl2}
	
	Under the conditions of Proposition C.1, the update operator $\mathcal{P}^{\text{Nash}}$ of Nash Q-learning algorithm forms a contraction mapping from $\mathcal{Q}$ to $\mathcal{Q}$. The fixed point of this algorithm is the Nash Q-value of the entire game. The update rule is
	$$
	\mathcal{P}^{\mathrm{Nash}} \boldsymbol{Q}(s, \boldsymbol{a})=\mathbb{E}_{s^{\prime} \sim \mathcal{T}_s}\left[\boldsymbol{r}(s, \boldsymbol{a})+\gamma \boldsymbol{v}^{\mathrm{Nash}}\left(s^{\prime}\right)\right]~.
	$$
\end{lemma}

\begin{proof}
	
	See Theorem 17 in \cites{hu2003nash}.
\end{proof}

\begin{lemma} \label{le: cl3}
	Let $\mathcal{P}$ be the operator that is used in multi agent soft q learning:
	$$
	\boldsymbol{Q}_{t+1}=\boldsymbol{r}_{t}+\gamma \mathbb{E}_{s^{\prime}\sim \mathcal{T}_t}[V_{\text{soft}}(s^{\prime})]=\boldsymbol{r}_{t}+\gamma \mathbb{E}_{s^{\prime}\sim \mathcal{T}_t}\Big[\alpha\log\int\exp(\boldsymbol{Q}_t/\alpha)d\boldsymbol{a}\Big]~.
	$$
	Then $\mathcal{P}$ forms a contraction mapping: $\|\mathcal{P}(\boldsymbol{Q}_1)-\mathcal{P}(\boldsymbol{Q}_2)\|_{\infty} \leq \gamma\|\boldsymbol{Q}_1-\boldsymbol{Q}_2\|_{\infty}$, where the norm $\|\cdot\|_{\infty}$ as $\|\boldsymbol{Q}_1-\boldsymbol{Q}_2\|_{\infty}=\max_{s, \boldsymbol{a}}|\boldsymbol{Q}_1-\boldsymbol{Q}_2|$.
\end{lemma}

\begin{proof}
	
	Suppose $\boldsymbol{\xi}=\|\boldsymbol{Q}_1-\boldsymbol{Q}_2\|_{\infty}$, then
	\begin{align*}
	\alpha\log\int\exp(\boldsymbol{Q}_1/\alpha)d\boldsymbol{a} & \leq \alpha\log\int\exp(\boldsymbol{Q}_2/\alpha+\boldsymbol{\xi}/\alpha)d\boldsymbol{a} \\
	& = \alpha\log \Big(\exp(\boldsymbol{\xi}/\alpha)\int\exp(\boldsymbol{Q}_2/\alpha)d\boldsymbol{a}\Big) \\
	& = \boldsymbol{\xi} + \alpha\log\int\exp(\boldsymbol{Q}_2/\alpha)d\boldsymbol{a}~.
	\end{align*}
	
	Thus, we have
	\begin{align*}
	\|\mathcal{P}(\boldsymbol{Q}_1)-\mathcal{P}(\boldsymbol{Q}_2)\|_{\infty} & = \|\gamma\mathbb{E}_{s^{\prime}\sim \mathcal{T}_s}[\alpha\log\int\exp(\boldsymbol{Q}_1/\alpha)d\boldsymbol{a}-\alpha\log\int\exp(\boldsymbol{Q}_2/\alpha)d\boldsymbol{a}]\|_{\infty} \\
	& \leq \gamma\mathbb{E}_{s^{\prime}\sim \mathcal{T}_s}[\boldsymbol{\xi}] \\
	& = \gamma\|\boldsymbol{Q}_1-\boldsymbol{Q}_2\|_{\infty}~,
	\end{align*}
	where $\mathcal{P}$ is a contraction mapping on Q values.
\end{proof}

\begin{lemma} \label{le: cl4}
	
	For a certain agent, given a policy $\pi$, define a new policy $\tilde{\pi}$ as
	$$
	\tilde{\pi}(\cdot|s) \propto \exp(Q_{\text{soft}}^{\pi}(s, \cdot)),\qquad \forall{s}~.
	$$
	
	Assume that throughout our computation, $Q$ is bounded and $\int\exp(Q)d\boldsymbol{a}$ is bounded for any $s$ (for both $\pi$ and $\tilde{\pi}$). Then $Q^{\tilde{\pi}}_{\text{soft}} \geq Q^{\pi}_{\text{soft}} \quad \forall{s, \boldsymbol{a}}$.
	
\end{lemma}

\begin{proof}
	The proof of this lemma is based on the observation that when update $\pi$ to $\tilde{\pi}$:
	$$
	\mathcal{H}(\pi(\cdot | s))+\mathbb{E}_{\boldsymbol{a} \sim \pi}\left[Q_{\text{soft}}^{\pi}(s, \boldsymbol{a})\right] \leq \mathcal{H}(\tilde{\pi}(\cdot | s))+\mathbb{E}_{\boldsymbol{a} \sim \tilde{\pi}}\left[Q_{\text{soft}}^{\pi}(s, \boldsymbol{a})\right]~.
	$$
	
	We can show that:
	\begin{align*}
	Q_{\mathrm{soft}}^{\pi}(\mathbf{s}, \mathbf{a}) & =\mathbb{E}_{\mathbf{s}_{1}}\left[r_{0}+\gamma\left(\mathcal{H}\left(\pi\left(\cdot | \mathbf{s}_{1}\right)\right)+\mathbb{E}_{\mathbf{a}_{1} \sim \pi}\left[Q_{\mathrm{soft}}^{\pi}\left(\mathbf{s}_{1}, \mathbf{a}_{1}\right)\right]\right)\right] \\
	& \leq \mathbb{E}_{\mathbf{s}_{1}}\left[r_{0}+\gamma\left(\mathcal{H}\left(\tilde{\pi}\left(\cdot | \mathbf{s}_{1}\right)\right)+\mathbb{E}_{\mathbf{a}_{1} \sim \tilde{\pi}}\left[Q_{\mathrm{soft}}^{\pi}\left(\mathbf{s}_{1}, \mathbf{a}_{1}\right)\right]\right)\right] \\
	& = \mathbb{E}_{\mathbf{s}_{1}}\left[r_{0}+\gamma\left(\mathcal{H}\left(\pi\left(\cdot | \mathbf{s}_{1}\right)\right)+r_1 \right)\right]+\gamma^2\mathbb{E}_{s_2}\left[\mathcal{H}(\pi(\cdot|s_2))+\mathbb{E}_{\mathbf{a}_{2} \sim \pi}\left[Q_{\mathrm{soft}}^{\pi}\left(\mathbf{s}_{2}, \mathbf{a}_{2}\right)\right]\right] \\
	& \leq \mathbb{E}_{\mathbf{s}_{1}}\left[r_{0}+\gamma\left(\mathcal{H}\left(\tilde{\pi}\left(\cdot | \mathbf{s}_{1}\right)\right)+r_1 \right)\right]+\gamma^2\mathbb{E}_{s_2}\left[\mathcal{H}(\tilde{\pi}(\cdot|s_2))+\mathbb{E}_{\mathbf{a}_{2} \sim \tilde{\pi}}\left[Q_{\mathrm{soft}}^{\tilde{\pi}}\left(\mathbf{s}_{2}, \mathbf{a}_{2}\right)\right]\right] \\
	& = \mathbb{E}_{s_1, a_2\sim\tilde{\pi},s_2}[r_{0}+\gamma\left(\mathcal{H}\left(\tilde{\pi}\left(\cdot | \mathbf{s}_{1}\right)\right)+r_1 \right)+\gamma^2(\mathcal{H}(\tilde{\pi}(\cdot|s_2))+r_2)] \\
	& \quad +\gamma^3\mathbb{E}_{\mathbf{a}_{3} \sim \tilde{\pi}}\left[\mathcal{H}(\tilde{\pi}(\cdot|s_3))+Q_{\mathrm{soft}}^{\tilde{\pi}}\left(\mathbf{s}_{3}, \mathbf{a}_{3}\right)\right] \\
	& \quad \vdots \\
	& \leq \mathbb{E}_{\tau\sim\tilde{\pi}}\Big[r_0+\sum_{t=1}^{\infty}\gamma^t(\mathcal{H}(\tilde{\pi}(\cdot|s_t))+r_t)\Big] \\
	& = Q^{\tilde{\pi}}_{\text{soft}}~.
	\end{align*}

	Thus for a certain agent, its Q value function increases monotonously.
\end{proof}

\begin{lemma} \label{le: cl5}
	
	Let $\mathcal{X}$ be an arbitrary set and assume that $\mathcal{B}$ is the space of bounded functions over $\mathcal{X}$. Denote $\mathcal{P}$ as an operator, i.e., $\mathcal{P}: \mathcal{B(X)}\to\mathcal{B(X)}$. Let $Q_{*}$ be a fixed point of $\mathcal{P}$ and let ${\mathcal{P}_0, \mathcal{P}_1, \ldots}$ approximate $\mathcal{P}$ at $Q_{*}$. Define $\mathcal{F}_t$ as the $\sigma$-field including the historical trajectories. Let $Q_0 \in \mathcal{F}_0$, and define $Q_{t+1} = \mathcal{P}(Q_t, Q_t)$. If there exist random functions $0 \leq F_{t}(x) \leq 1$ and $0 \leq G_{t}(x) \leq 1$ satisfying the conditions below w.p.1, then $Q_t$ converges to $Q_*$ w.p.1:
	\begin{enumerate}[leftmargin=15px]
		\item For all $x \in \mathcal{X}, Q_1, Q_2 \in \mathcal{F}_1$,
		$$
		\left|T_{t}\left(Q_1, Q^{*}\right)(x)-T_{t}(Q_2,  Q^{*})(x)\right| \leq G_t(x)|Q_1(x)-Q_2(x)|~;
		$$
		\item For all $x \in \mathcal{X}, Q_1, Q_2 \in \mathcal{F}_t$,
		$$
		\left|T_{t}\left(Q_1, Q^{*}\right)(x)-T_{t}(Q_1, Q_2)(x)\right| \leq F_{t}(x)\left(\left\|Q^{*}-Q_2\right\|_{\infty}+\lambda_{t}\right)~,
		$$
		where $\lambda_t\to 0$ w.p.1 as $t\to\infty$.
		\item For all $k>0, \Pi_{t=k}^{n} G_{t}(x)$ converges to $0$ uniformly in $x$ as $n\to\infty$.
		\item There exists $0 \leq \gamma \leq 1$ such that for all $x \in \mathcal{X}$ and large enough $t$, we have
		$$
		F_{t}(x) \leq \gamma\left(1-G_{t}(x)\right)~.
		$$
	\end{enumerate}
\end{lemma}

\begin{proof}
	See Proposition 3 in \cites{szepesvari1999unified}.
\end{proof}

\begin{lemma} \label{le: cl6}
	For any opponent model $\boldsymbol{\phi}$, operator $\oplus_t^{\boldsymbol{\phi}}$ forms a contraction mapping with the norm $\|\cdot\|_\infty$ as defined in Lemma D.3.
	\begin{align*}
	\|\oplus_t^{\boldsymbol{\phi}}\boldsymbol{f}(\cdot)-\oplus_t^{\boldsymbol{\phi}}\boldsymbol{g}(\cdot)\|_{\infty} & = \Big\|\mathbb{E}_{s^{\prime}\sim \mathcal{T}_t}\Big[\alpha\log\int_{\boldsymbol{A}^{-i}}\boldsymbol{\phi}(\boldsymbol{a}^{-i}|s^{\prime})\int_{\boldsymbol{A}^i}\exp(\boldsymbol{f}(\cdot)/\alpha)d\boldsymbol{a}^{i}d\boldsymbol{a}^{-i} \\
	& \quad\quad\quad\quad\quad -\alpha\log\int_{\boldsymbol{A}^{-i}}\boldsymbol{\phi}(\boldsymbol{a}^{-i}|s^{\prime})\int_{\boldsymbol{A}^i}\exp(\boldsymbol{g}(\cdot)/\alpha)d\boldsymbol{a}^{i}d\boldsymbol{a}^{-i}\Big]\Big\|_{\infty} \\
	& \leq \|\boldsymbol{f}-\boldsymbol{g}\|_{\infty}~.
	\end{align*}
	
\end{lemma}

\begin{proof}
	This lemma can be proved using the same method that is used in Lemma D.3.
	Let $\boldsymbol{\xi}=\|\boldsymbol{f}-\boldsymbol{g}\|_{\infty}$, we have
	\begin{align*}
	\alpha\log\int_{\boldsymbol{A}^{-i}} & \boldsymbol{\phi}(\boldsymbol{a}^{-i}|s^{\prime})\int_{\boldsymbol{A}^i}\exp(\boldsymbol{f}(\cdot)/\alpha)d\boldsymbol{a}^{i}d\boldsymbol{a}^{-i} \\
	& \leq \alpha\log\int_{\boldsymbol{A}^{-i}}\boldsymbol{\phi}(\boldsymbol{a}^{-i}|s^{\prime})\int_{\boldsymbol{A}^i}\exp(\boldsymbol{g}(\cdot)/\alpha+\boldsymbol{\xi}/\alpha)d\boldsymbol{a}^{i}d\boldsymbol{a}^{-i} \\
	& = \alpha\log\int_{\boldsymbol{A}^{-i}}\boldsymbol{\phi}(\boldsymbol{a}^{-i}|s^{\prime})\int_{\boldsymbol{A}^i}\exp(\boldsymbol{g}(\cdot)/\alpha)d\boldsymbol{a}^{i}d\boldsymbol{a}^{-i} + \boldsymbol{\xi}~.
	\end{align*}
	Then we can show that
	$$
	\begin{aligned}
	\|\oplus_t^{\boldsymbol{\phi}}\boldsymbol{f}(\cdot)-\oplus_t^{\boldsymbol{\phi}}\boldsymbol{g}(\cdot)\|_{\infty}
	\leq \|\mathbb{E}_{s^{\prime}\sim \mathcal{T}_t}[\boldsymbol{\xi}]\|_{\infty}
	= \|\boldsymbol{f}-\boldsymbol{g}\|_{\infty}
	\end{aligned}~.
	$$
\end{proof}


\section{The \methodname{} Method Building Blocks}\label{app:building-blocks}
This section introduces the details of key components of our solution as briefly mentioned in Section~\ref{subs: method}, including the methods of the dynamics model learning, and the opponent modeling. 
The implementation of these key components is based on previous work, which serves as the preliminaries.

\minisection{Learning the Dynamics Model} A bootstrap ensemble of probabilistic dynamics models $\{\hat{\mathcal{T}}_{\theta^{1}}, \ldots, \hat{\mathcal{T}}_{\theta^{B}}\}$ 
is used to predict the environment dynamics in \cites{chua2018deep}. In detail, each $\theta^{b}$-parameterized dynamics model outputs a Gaussian distribution of the next state as $\hat{\mathcal{T}}_{\theta^{b}}(s^{\prime} | s, a^{i}, a^{-i})=\mathcal{N}({\mu}_{\theta^{b}}({s}, a^{i}, a^{-i}),$ 
$\Sigma_{\theta^{b}}({s}, a^{i}, a^{-i}))$. 
Individual probabilistic dynamics model can capture the aleatoric uncertainty arisen from inherent stochasticity of a system and the bootstrap ensemble is used to capture the epistemic uncertainty, due to a lack of sufficient data \cite{chua2018deep}. 
Previous works have demonstrated that this kind of model leads to better asymptotic performance than single deterministic model, even when the ground-truth dynamic is deterministic \cite{lakshminarayanan2017simple}. 
We train the  models with all the data obtained from real environment via maximum likelihood estimation. The loss for each dynamics model is
\begin{align}
        J_{\hat{\mathcal{T}}}({\theta^{b}}) = \mathbb{E}_{s_t, a_t^{i}, a_t^{-i}, s_{t+1} \sim \mathcal{D}}
        \Big[
        & \log \operatorname{det} {\Sigma}_{\theta^{b}}({s}_{t}, a_t^{i}, a_t^{-i})  
         +e^{\top}{\Sigma}_{\theta^{b}}^{-1}({s}_{t}, a_t^{i}, a_t^{-i})e 
        \Big]~, \nonumber
\end{align}
where $e=\mu_{\theta^{b}}({s}_{t}, a_t^{i}, a_t^{-i})-{s}_{t+1}~.$

When generating a prediction from the probabilistic ensemble dynamics model, we select one of the models randomly from a uniform distribution, which means that the transitions along one rollout can be sampled from different individual dynamics model. 
For simplicity, we denote the ensemble model by $\hat{\mathcal{T}}_{\theta}(s^{\prime} | s, a^{i}, a^{-i})$.

\minisection{Learning the Opponent Models} For agent $i$, the policy of one opponent agent $j$ can be modeled as a Gaussian distribution of the action: $\hat{\pi}^{j}(a^{j}|s)=\pi_{\phi^{j}}(a^{j}|s)=\mathcal{N}({\mu}_{\phi^{j}}(s), {\Sigma}_{\phi^{j}}(s))$, where $\phi_j$ is the parameter of the opponent model for the agent $j$. 
We use maximum likelihood estimation of the opponent $j$'s actions with an entropy regularizer $H(\pi_{\phi^{j}}(\cdot|s_t))$ to learn the opponent's policy. The loss for each opponent model as
\begin{align}
        J_{\pi}({\phi^{j}})  &= -\mathbb{E}_{s_t, a_t^{j} \sim \mathcal{D}}\left[\log(\pi_{\phi^{j}}(a_{t}^{j}|s_t))\right]+\alpha H(\pi_{\phi^{j}}(\cdot|s_t)) \nonumber  \\
        & =\sum_{t=1}^{N}\big[\mu_{\phi^{j}}({s}_{t})-a_{t}^{j}\big]^{\top} {\Sigma}_{\phi^{j}}^{-1}({s}_{t})\big[\mu_{\phi^{j}}\left({s}_{t}\right)-a_{t}^{j}\big] \label{eq: opponent}
         +\log \operatorname{det} {\Sigma}_{\phi^{j}}\left({s}_{t}\right)+\alpha H(\pi_{\phi^{j}}(\cdot|s_t))~,\nonumber    
\end{align}
where $\mathcal{D}$ is the training set of the latest collected samples.

In line \ref{alg: opp_err} of \alg{alg}, the opponent model error $\epsilon_{\hat{\pi}}^{j}$ is computed for each opponent model. For the continuous action space, we define 
$\epsilon_{\hat{\pi}}^{j} = \mathbb{E}_{s_t, a_t^j\sim \mathcal{D}^{v}, \hat{a}_t^j\sim \pi_{\phi^j}(\cdot|s_t)}[\|a_t^j-\hat{a}_t^j\|_2]$; 
for the discrete action space, we define 
$\epsilon_{\hat{\pi}}^{j} = \mathbb{E}_{s_t, a_t^j\sim \mathcal{D}^{v}, \hat{a}_t^j\sim \pi_{\phi^j}(\cdot|s_t)}[1-\mathds{1}(a_t^j=\hat{a}_t^j)]$, where $\mathcal{D}^{v}$ is the the evaluation set of the latest collected samples.


\section{Implementation Details and Hyperparameter Settings}
\label{app: hyper_set}

We implement all our experiments using PyTorch. The action spaces in all our experiments except Climb are discrete. Unless otherwise specified, our policies and opponent models are parameterized by a three-layer MLP with ReLU as activation function, and each layer has 64 units, and our dynamics models consist of four MLP with swish \cite{ramachandran2017searching} as activation function, where each layer has 256 units. To support discrete action space, we use the Gumbel-Softmax estimator \cite{jang2016categorical} to get discrete actions. We use Adam optimizer to update the parameters. The reward discount factor is set to 0.95, the size of the replay buffer to 500,000 and the batch size to 1,024. Soft updates with target networks use ratio as 0.01. The learning rate used in updating dynamics model is halved every 5,000 episodes. The number of timesteps for each episode is set to 25. We use accuracy for discrete action space and mean square error for continuous action space, to measure the prediction performance of the opponent models. More details of the important hyperparameters settings of \methodname{} are listed in \tab{tab: hypers-1} for cooperative tasks and in \tab{tab: hypers-2} for competitive tasks, respectively.

The results are obtained by running 10 random seeds. In the definition of \textit{normalized agent interactions}, a full interaction among an agent group means every agent takes an action.

In this paper, the computing resources we use to conduct experiments are several CPUs and GPUs. In detail, the CPUs are Intel Xeon E5-2686 v4 @ 2.30GHz and the GPUs are NVIDIA GeForce RTX 2080 Ti.

\begin{table}[t]
	\centering
	\begin{tabular}{|c|c|c|c|c|}
		\hline
		Parameter & Description & \makecell*[c]{Cooperative \\ communication} & \makecell*[c]{Cooperative \\ navigation} & \makecell*[c]{Cooperative \\ scheduling} \\
		\hline\hline
		$N$ & Epochs & \multicolumn{3}{c|}{200} \\
		\hline
		$E$ & \makecell*[c]{Environment steps \\ between model updates} & 200 & \multicolumn{2}{c|}{300} \\
		\hline
		$M$ & \makecell*[c]{Model rollout \\ batch size} & \multicolumn{3}{c|}{1024} \\
		\hline
		$k$ & \makecell*[c]{Branched rollout \\ length} & \makecell*[c]{1$\to$10\\over epochs\\15$\to$100} & \makecell*[c]{1$\to$6\\over epochs\\15$\to$100} & \makecell*[c]{1$\to$8\\over epochs\\40$\to$100} \\
		\hline
		$B$ & \makecell*[c]{Dynamics model \\ ensemble size} & \multicolumn{2}{c|}{10} & 8 \\
		\hline
		$G$ & \makecell*[c]{Policy update steps \\ per environment step} & 10 & \multicolumn{2}{c|}{20} \\
		\hline
		- & Policy learning rate & 0.002 & 0.01 & 0.01 \\
		\hline
		- & \makecell*[c]{Opponent models \\ learning rate} & 0.001 & 0.0003 & 0.0005 \\
		\hline
		- & \makecell*[c]{Dynamics models \\ learning rate} & 0.003 & 0.001 & 0.0025 \\
		\hline
	\end{tabular}
	\vspace{5pt}
	\\
	Note: $x\to y$ over epochs $a\to b$ denotes a thresholded linear function, i.e., at the $e$th epoch , $f(e)=\min (\max (x+\frac{e-a}{b-a} \cdot(y-x), x), y)$.
	\vspace{5pt}
	\caption{The hyperparameters used in the experiments of the cooperative tasks.}
	\label{tab: hypers-1}
\end{table}

\begin{table}[!htbp]
	\centering
	\begin{tabular}{|c|c|c|c|c|}
		\hline
		Parameter & Description & \makecell*[c]{Keep-away} & \makecell*[c]{Physical deception} & \makecell*[c]{Predator prey} \\
		\hline\hline
		$N$ & Epochs & \multicolumn{3}{c|}{200} \\
		\hline
		$E$ & \makecell*[c]{Environment steps \\ between model updates} & 50 & \multicolumn{2}{c|}{200} \\
		\hline
		$M$ & \makecell*[c]{Model rollout \\ batch size} & \multicolumn{3}{c|}{1024} \\
		\hline
		$k$ & \makecell*[c]{Branched rollout \\ length} & \makecell*[c]{1$\to$10\\over epochs\\30$\to$120} & \makecell*[c]{1$\to$8\\over epochs\\20$\to$100} & \makecell*[c]{1$\to$5\\over epochs\\20$\to$100} \\
		\hline
		$B$ & \makecell*[c]{Dynamics model \\ ensemble size} & \multicolumn{2}{c|}{8} & 10 \\
		\hline
		$G$ & \makecell*[c]{Policy update steps \\ per environment step} & \multicolumn{3}{c|}{20} \\
		\hline
		- & Policy learning rate & 0.01 & 0.0004 & 0.0001 \\
		\hline
		- & \makecell*[c]{Opponent models \\ learning rate} & 0.0005 & 0.005 & 0.002 \\
		\hline
		- & \makecell*[c]{Dynamics models \\ learning rate} & 0.001 & 0.0005 & 0.0005 \\
		\hline
	\end{tabular}
	\vspace{5pt}
	\caption{The hyperparameters used in the experiments of the competitive tasks.}
	\label{tab: hypers-2}
\end{table}

\section{More Experiment Results}
\label{app: exp}

\subsection{Task Descriptions}

\minisection{Cooperative communication} scenario with two agents and three landmarks. One of the agents is a speaker that knows the target landmark but cannot move, and the other is a listener that can move to the target landmark via communication with the speaker.

\minisection{Cooperative navigation} scenario with three agents and three landmarks. Agents are collectively rewarded based on the proximity of any agent to each landmark while avoiding collision. 

\begin{wrapfigure}{r}{0.35\textwidth}
  \begin{center}
  \vspace{-30pt}
	\includegraphics[width=1\linewidth]{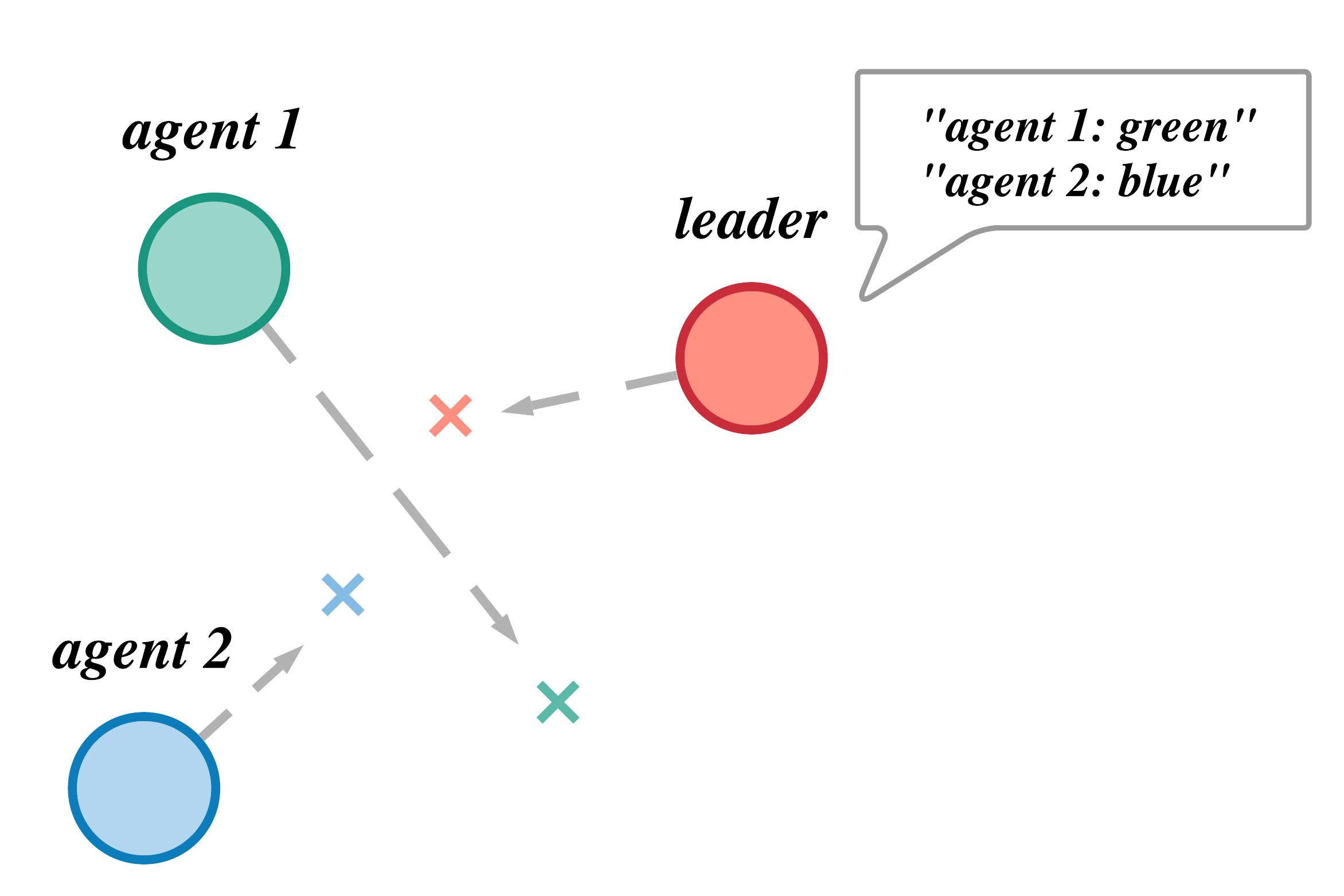}
  \end{center}
  \caption{Cooperative scheduling task.}\vspace{-15pt}
  \label{fig:scheduling}
\end{wrapfigure}

\minisection{Cooperative scheduling} scenario is described in \se{sec: exp}, as shown in \fig{fig:scheduling}. Agents are collectively rewarded based on the proximity of each agent to its corresponding landmark. There is one leader that knows the correspondence between agents and landmarks, the other two agents that need to know their corresponding landmarks via communication with the leader. The agents do not collide in this task. We implement the model-free algorithms with opponent modeling in this task.

\minisection{Climb} The agents have different preferences in two states.  At state 1, agent 1 maximizes the reward by moving to the lower left landmark. By contrast, agent 2 can get different rewards by moving to both landmarks, computed as $\max(-2(a+0.5)^2-2(b+0.5)^2, -(a-0.5)^2-(b-0.5)^2)$. At state 2, the agents' preferences change, agent 1 maximizes the reward by moving to the upper right landmark. By contrast, agent 2 can get different rewards by moving to both landmarks, computed as $\max(-(a+0.5)^2-(b+0.5)^2, -2(a-0.5)^2-2(b-0.5)^2)$.

\minisection{Keep-away} scenario with one agent, one adversary and two landmarks. The agent knows the target landmark and is rewarded based on the distance to the target landmark. The adversary is rewarded for preventing the agent from reaching the target landmark. The adversary does not know the correct landmark and must infer it from the movement of the agents.

\minisection{Physical deception} scenario with two agents, one adversary and two landmarks. The agents know the target landmark and are rewarded based on the minimum distance of any agent to the target landmark. The adversary does not know the target landmark and needs to reach the target landmark by following the agents. The agents are also rewarded by how well they deceive the adversary.

\minisection{Predator-prey} scenario with one prey (agent) and three predators (adversaries). The prey moves faster and is rewarded for running away from the three predators. The predators move slower and need to catch the prey through cooperation.

\subsection{Cooperative Task Results}
The Cooperative scheduling is set as a SECO task: (i) the agents do not collide, which makes the dynamics models prediction easier; (ii) the agents have heterogeneous action spaces, which make it more challenging for the opponent models to predict the opponents' actions accurately. In the right subfigure of \fig{fig: cooperative}, the sample efficiency refers to the amount of the interactions between the agents. For fair comparison, we implement the model-free baselines with opponent models in this experiment. Our methods \methodname{} and \methodnametwo{} reach the asymptotic performance with less agents' interactions, which indicates high sample efficiency. 
Similar to the results for CESO tasks, our methods contribute to the more stable training process.

\subsection{Competitive Task Results}
\label{subs: tab_res}

We evaluate the sample efficiency of our method in the competitive environments on the Multi-Agent Particle Environment \cite{lowe2017multi}. Three scenarios are included, i.e., 1) \textit{Keep-away} scenario with an agent, an adversary and two landmarks, 2) \textit{Physical deception} scenario with two agents, an adversary, and two landmarks, and 3) \textit{Predator-prey} scenario with a prey (agent) and three predators (adversaries).

In these competitive tasks, we focus on the coordination within a team: the ego agent does not communicate with agents from other teams. While conducting the model rollouts, the ego agent $i$ communicates with only the agents of its own team, and samples actions of agents in the other teams with the opponent models.

\begin{figure}[!ht]
	\centering
	\includegraphics[width=0.65\linewidth]{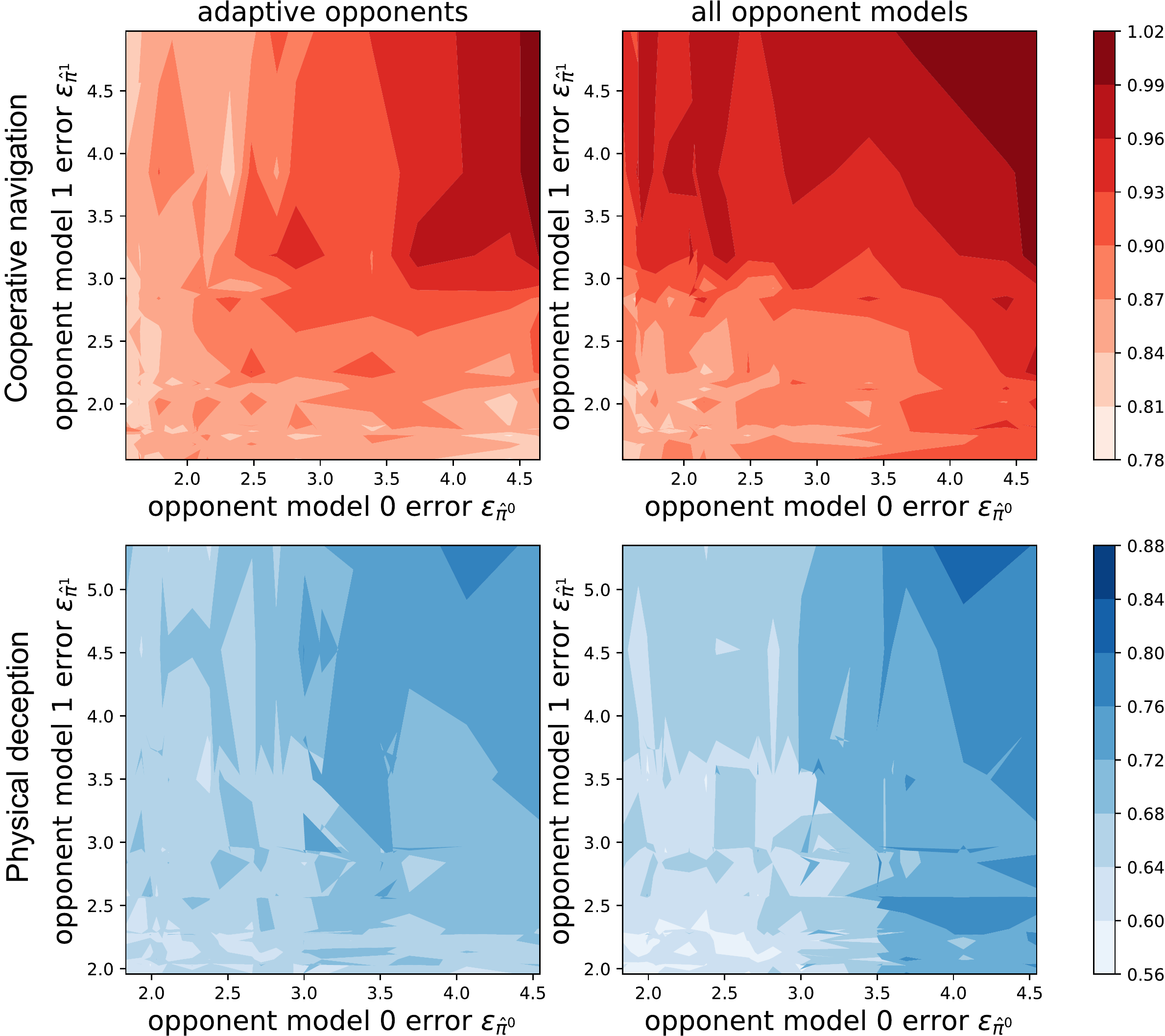}
	\caption{Model compounding errors in Cooperative navigation and Physical deception. The figures show the compounding errors (the dark color, the higher error) using different opponent model usages with opponent models that have different errors.
	}
	\label{fig: comp_c_e}
\end{figure}

\minisection{Experiment Protocol}
Each bar in Figure \ref{fig: competition} shows the $[0,1]$ normalized agent reward score ($\textit{agent~reward} - \textit{adversary~reward}$). The higher score, the better for the agent. 
All the results are produced by evaluating every task 30 times with different random seeds.

\begin{figure*}[t]
	\centering
	\includegraphics[width=1\linewidth]{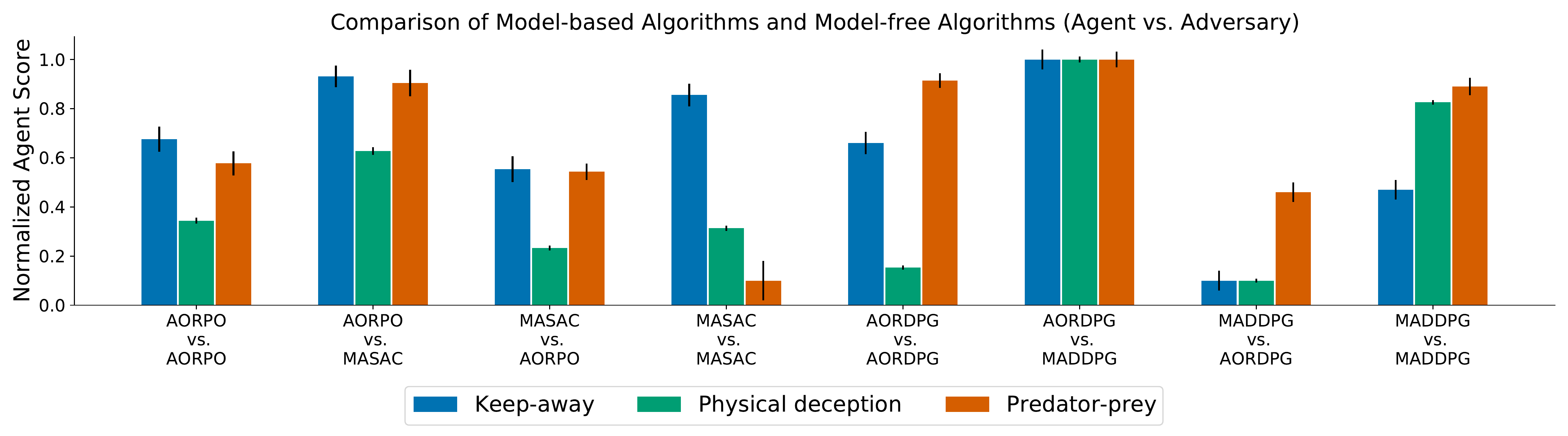}
	\caption{Performance comparison on the competition environments in Multi-Agent Particle Environment. Each bar shows the $[0,1]$ normalized agent score (agent reward minus adversary reward).
	} 
	\label{fig: competition}
\end{figure*}

Since the choice of the base model-free algorithms affects the final performance in the competitive tasks, we compare \methodname{} and \methodnametwo{} with MASAC and MADDPG, respectively, following the evaluation protocol in \cites{lowe2017multi}. The comparison results are shown in \fig{fig: competition}. We find that \methodname{} gets higher agent scores (agent reward minus adversary reward) than MASAC does when competing with the same rivals and that \methodnametwo{} outperforms MADDPG similarly.

\minisection{More Detailed Performance}
The rows are the algorithms that used by the adversaries, and the columns are the ones used by the agents, i.e., the score of each entry corresponds to the performance of the column algorithm. Results are in Table \ref{tab: com_1} and \ref{tab: com_2}.

\begin{table}[!htbp]
	\small
	\centering
	\begin{subtable}{0.31\linewidth}
		\centering
		\caption{Results of Keep-away.}
		\begin{tabular}{l l l}
			\toprule
			Algorithms    & MASAC & \methodname{} \\
			\midrule
			MASAC         & -6.68 & \textbf{-6.28} \\
			\methodname{} & -8.31 & \textbf{-7.65} \\
			\bottomrule
		\end{tabular}
	\end{subtable}
	\begin{subtable}{0.31\linewidth}
		\centering
		\caption{Results of Physical deception.}
		\begin{tabular}{l l l}
			\toprule
			Algorithms    & MASAC & \methodname{} \\
			\midrule
			MASAC         & 19.64 & \textbf{32.49} \\
			\methodname{} & 16.36 & \textbf{20.89} \\
			\bottomrule
		\end{tabular}
	\end{subtable}
	\begin{subtable}{0.31\linewidth}
		\centering
		\caption{Results of Predator prey.}
		\begin{tabular}{l l l}
			\toprule
			Algorithms    & MASAC  & \methodname{} \\
			\midrule
			MASAC         & -74.93 & \textbf{-25.69} \\
			\methodname{} & -47.79 & \textbf{-45.70} \\
			\bottomrule
		\end{tabular}
	\end{subtable}
	\vspace{5pt}
	\caption{Results of \methodname{} vs.MASAC.}
	\label{tab: com_1}
\end{table}

\begin{table}[!htbp]
	\small
	\centering
	\begin{subtable}{0.31\linewidth}
		\centering
		\caption{Results of Keep-away.}
		\resizebox{\linewidth}{!}{
		\begin{tabular}{l l l}
			\toprule
			Algorithms           & MADDPG & \makecell*[c]{\methodnametwo{}} \\
			\midrule
			MADDPG               & 40.60  & \textbf{47.72} \\
			\makecell*[c]{\methodnametwo{}} & 10.93  & \textbf{13.12} \\
			\bottomrule
		\end{tabular}}
	\end{subtable}
	\begin{subtable}{0.31\linewidth}
		\centering
		\caption{Results of Physical deception.}
		\resizebox{\linewidth}{!}{
		\begin{tabular}{l l l}
			\toprule
			Algorithms           & MADDPG & \makecell*[c]{\methodnametwo{}} \\
			\midrule
			MADDPG               & -8.75  & \textbf{-5.91} \\
			\makecell*[c]{\methodnametwo{}} & -10.75 & \textbf{-7.73} \\
			\bottomrule
		\end{tabular}}
	\end{subtable}
	\begin{subtable}{0.31\linewidth}
		\centering
		\caption{Results of Predator prey.}
		\resizebox{\linewidth}{!}{
		\begin{tabular}{l l l}
			\toprule
			Algorithms           & MADDPG & \makecell*[c]{\methodnametwo{}} \\
			\midrule
			MADDPG               & -26.53 & \textbf{-19.83} \\
			\makecell*[c]{\methodnametwo{}} & -52.89 & \textbf{-25.04} \\
			\bottomrule
		\end{tabular}}
	\end{subtable}
	\vspace{5pt}
	\caption{Results of \methodnametwo{} vs. MADDPG.}
	\label{tab: com_2}
	
\end{table}

\subsection{Comparison with Other Model-based Algorithms}\label{sec:comp-with-mbmarl}

We compare our algorithm \methodname{} with the method called \textit{multi-agent RL with multi-step generative models} (MAMSGM) as proposed by \cite{krupnik2019multi}. We implement MAMSGM in the online form, which means that the method updates the models and collects data by planning the trajectories using the models. The experiments are conducted in the Predator prey scenario. The results are shown in Table \ref{tab: com_3}, the meanings of the tables are the same as the ones in \app{subs: tab_res}, i.e., the score of each entry corresponds to the performance of the column algorithm.

\begin{table}[!htbp]
	\small
	\centering
	\begin{subtable}{0.31\linewidth}
		\centering
		\caption{Results in Keep-away.}
		\begin{tabular}{l l l}
			\toprule
			Algorithms & MAMSGM & \makecell*[c]{\methodname{}} \\
			\midrule
			\makecell*[c]{MAMSGM} & -61.31  & \textbf{44.1} \\
			\methodname{} & -107.9  & \textbf{-7.65} \\
			\bottomrule
		\end{tabular}
	\end{subtable}
	\hfill
	\begin{subtable}{0.31\linewidth}
		\centering
		\caption{Results in Physical deception.}
		\begin{tabular}{l l l}
			\toprule
			Algorithms & MAMSGM & \makecell*[c]{\methodname{}} \\
			\midrule
			\makecell*[c]{MAMSGM} & 134.3  & \textbf{221.44} \\
			\methodname{} & 2.65  & \textbf{20.89} \\
			\bottomrule
		\end{tabular}
	\end{subtable}
	\hfill
	\begin{subtable}{0.31\linewidth}
		\centering
		\caption{Results in Predator prey.}
		\begin{tabular}{l l l}
			\toprule
			Algorithms & MAMSGM & \makecell*[c]{\methodname{}} \\
			\midrule
			\makecell*[c]{MAMSGM} & -69.03  & \textbf{-3.26} \\
			\methodname{} & -357.9  & \textbf{-45.69} \\
			\bottomrule
		\end{tabular}
	\end{subtable}
	\vspace{4pt}
	\caption{Comparison between \methodname{} and MAMSGM.}
	\label{tab: com_3}
\end{table}




\subsection{Model Error Analysis}\label{app:model-error}
In Figure \ref{fig: comp_r_l}, we have shown the relationship between the model compounding error and the interactions numbers of different rollout schemes.
Further, we investigate that the adaptive opponent-wise rollout method can reduce the model compounding error with diverse opponent models. As shown in Figure \ref{fig: comp_c_e}, using opponent models with different performance, the adaptive opponent model selection reduces the model compounding errors generally.


\subsection{More Analysis on Adaptive Opponent-wise Rollout} \label{app:two-rollout} 

There are at least two feasible implementations of adaptive opponent-wise rollout schemes, namely \textit{simu-to-real} rollout and \textit{real-to-simu} rollout, as illustrated in Figure~\ref{fig:two-rollout-illustration}. The one shown in Figure~\ref{fig: ma-mbrl} is the simu-to-real rollout. In this section, we discuss the pros and cons of these two schemes. 

\begin{figure}[t]
	\centering
	\includegraphics[width=0.75\linewidth]{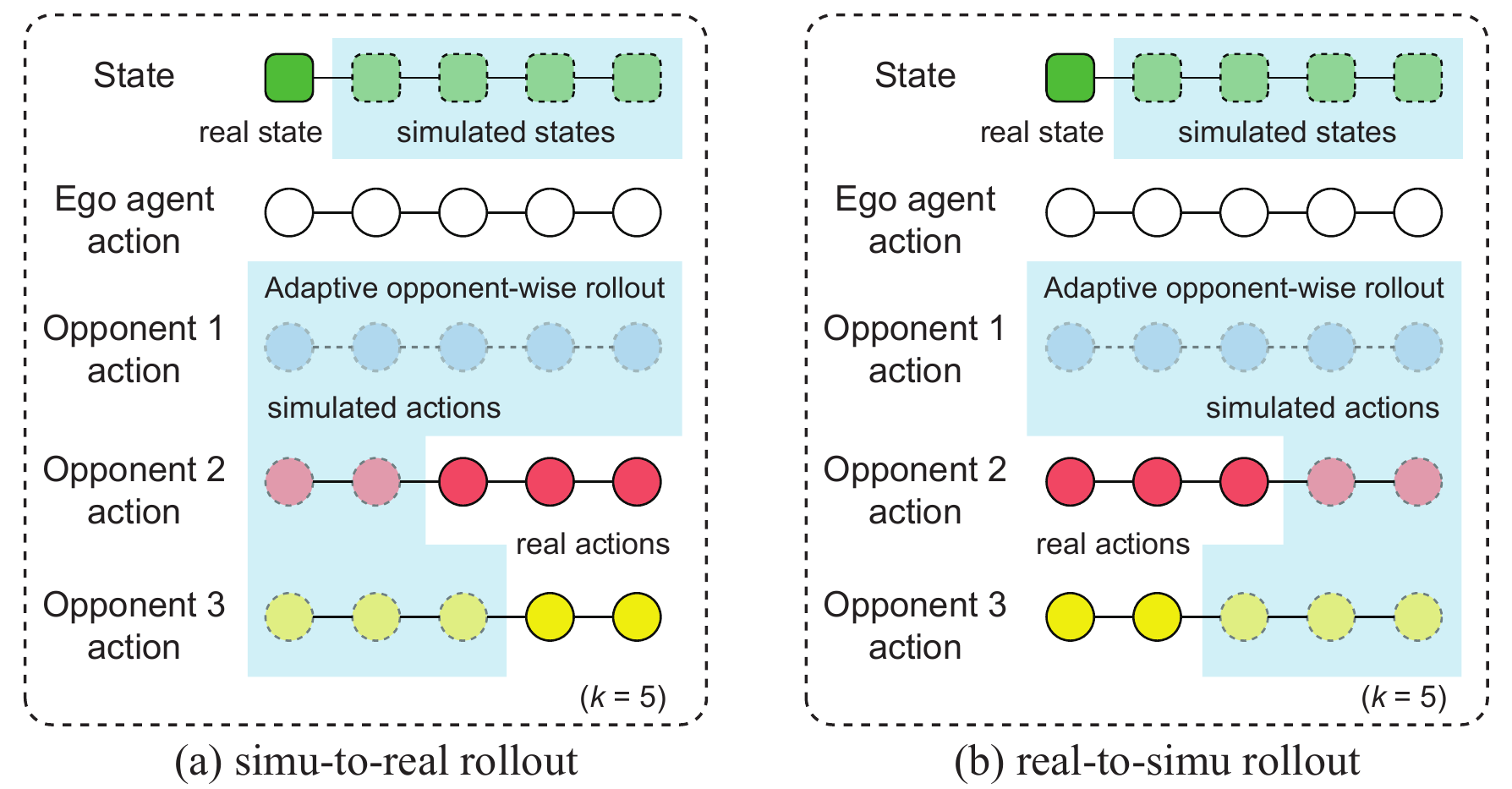}
	\caption{Two feasible adaptive opponent-wise rollout schemes. (a) simu-to-real rollout: first use opponent models to simulate the actions, then call the opponent agents for real actions; (b) real-to-simu rollout: first call the opponent agents for real actions, then use opponent models to simulate more actions.
	}
	\label{fig:two-rollout-illustration}
\end{figure}

We first analyze the difference of the state distance between these two alternatives. For simplicity, we consider the situation that there are only two opponents, and without loss of generality we assume the total branched rollout length using the opponent model $\hat{\pi}^1$ with small error is $k$ and the length of the other opponent model $\hat{\pi}^2$ performed rollout is $l$.
We now briefly derive the state distance in the $k$ steps between the adaptive opponent-wise rollout and the scheme of using both real opponents in the dynamics model according to Extended Lemma~\ref{exle: bl2}. 
For simu-to-real rollout, the state distance of the $k$ steps is bounded by $ \sum_{t=1}^k t(\epsilon^1_{\hat{\pi}} + \epsilon^2_{\hat{\pi}}) + l(k-l) (\epsilon^1_{\hat{\pi}} + \epsilon^2_{\hat{\pi}}) + \sum_{t=l+1}^k (t-l) \epsilon^1_{\hat{\pi}}$ (the discount factor $\gamma$ is omitted for simplicity). 
And for real-to-simu rollout, the state distance of the $k$ steps is bounded by $\sum_{t=1}^{k-l} t \epsilon^1_{\hat{\pi}} + l(k-l)\epsilon^1_{\hat{\pi}} + \sum_{t=k-l+1}^k(t-k+l) (\epsilon^1_{\hat{\pi}} + \epsilon^2_{\hat{\pi}})$. By comparison, we find that the state distance in the $k$ steps of simu-to-real is larger than that of real-to-simu by $l(k-l)\epsilon^2_{\hat{\pi}}$, which means simu-to-real will compound more errors. This is easy to understand since in real-to-simu at first no error of the opponent model $\hat{\pi}^2$ is induced and thus will not be compound later. 


\begin{figure}[t]
	\centering
	\includegraphics[width=0.6\linewidth]{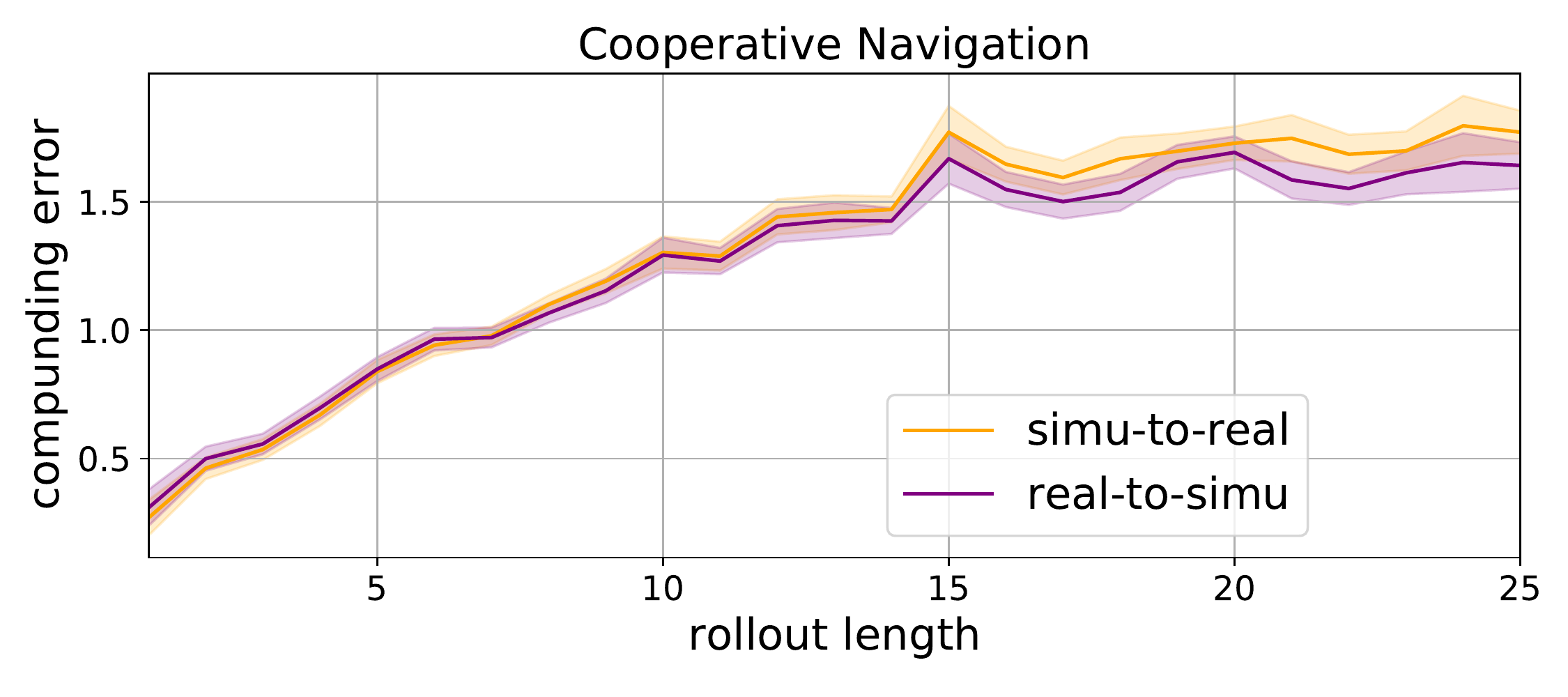}
	\caption{Model compounding errors and ratios for two rollout schemes.
	}
	\label{fig:two-rollout-model-error}
\end{figure}

\fig{fig:two-rollout-model-error} shows the results in \textit{Cooperative communication} scenario.
We can see that the real-to-simu rollout scheme achieves lower compounding error than simu-to-real scheme for long rollout (when $k>15$) but the model compounding errors are almost the same for short rollout (when $k < 15$). When $k=10$, the compounding errors are almost the same, while the contribution ratio of compounding error from opponent models from simu-to-real scheme is smaller than that from real-to-simu scheme.
Furthermore, as shown in \fig{fig:two-rollout-effectiveness} and \fig{fig:two-rollout-performance}, the simu-to-real scheme actually achieves higher sample efficiency than real-to-simu scheme, for both $k=10$ and $k=20$, which looks counter-intuitive.


\begin{minipage}[t]{0.48\linewidth}
	\vspace{-0pt}
	\centering
	\includegraphics[height=0.6\linewidth]{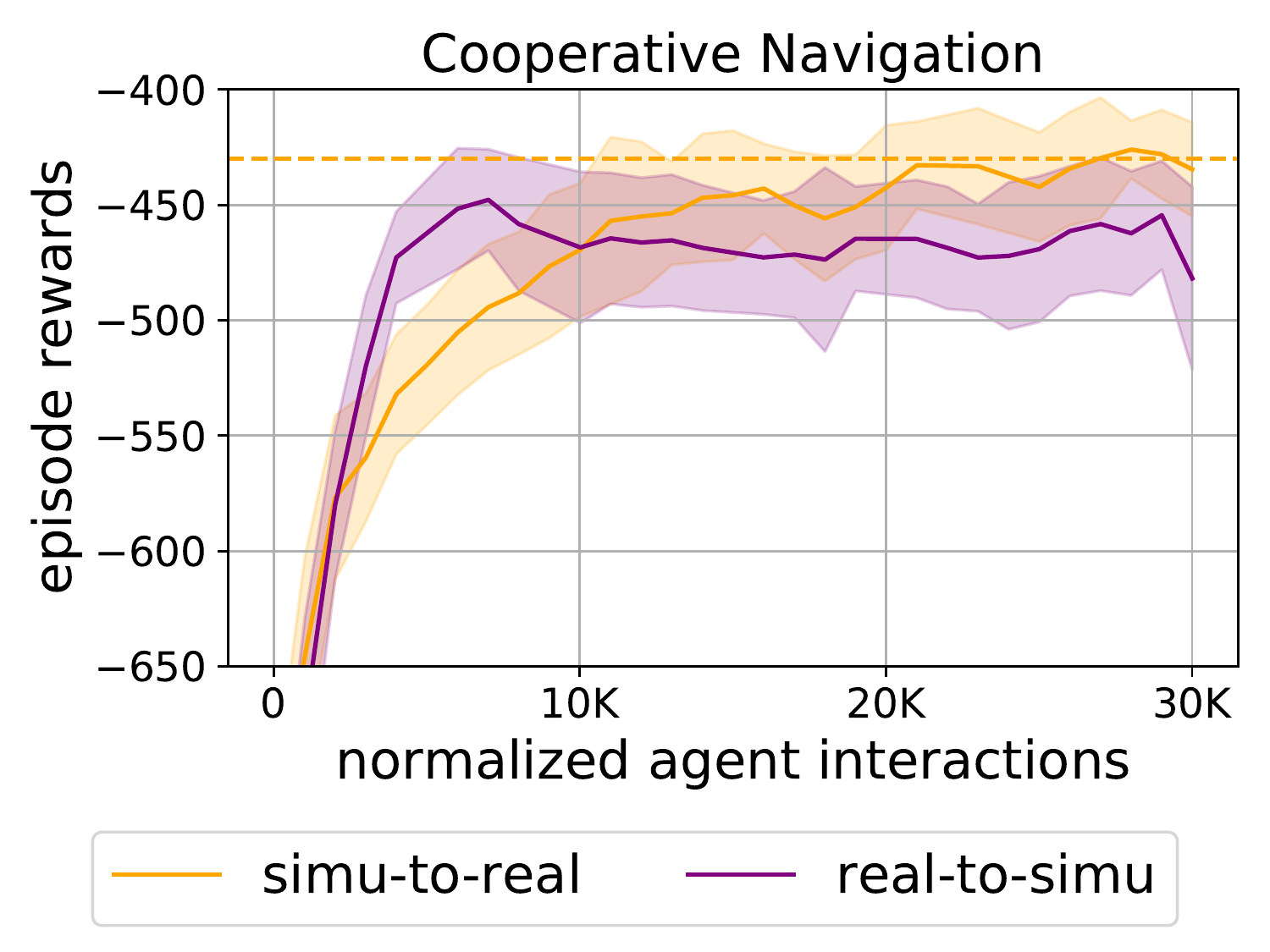}
	\captionof{figure}{Effectiveness w.r.t. interactions, using fixed learned dynamics model and $k=20$.
	}
	\label{fig:two-rollout-effectiveness}
	\vspace{10pt}
\end{minipage}
\hfill
\begin{minipage}[t]{0.48\linewidth}
	\vspace{-0pt}
	\centering
	\includegraphics[height=0.6\linewidth]{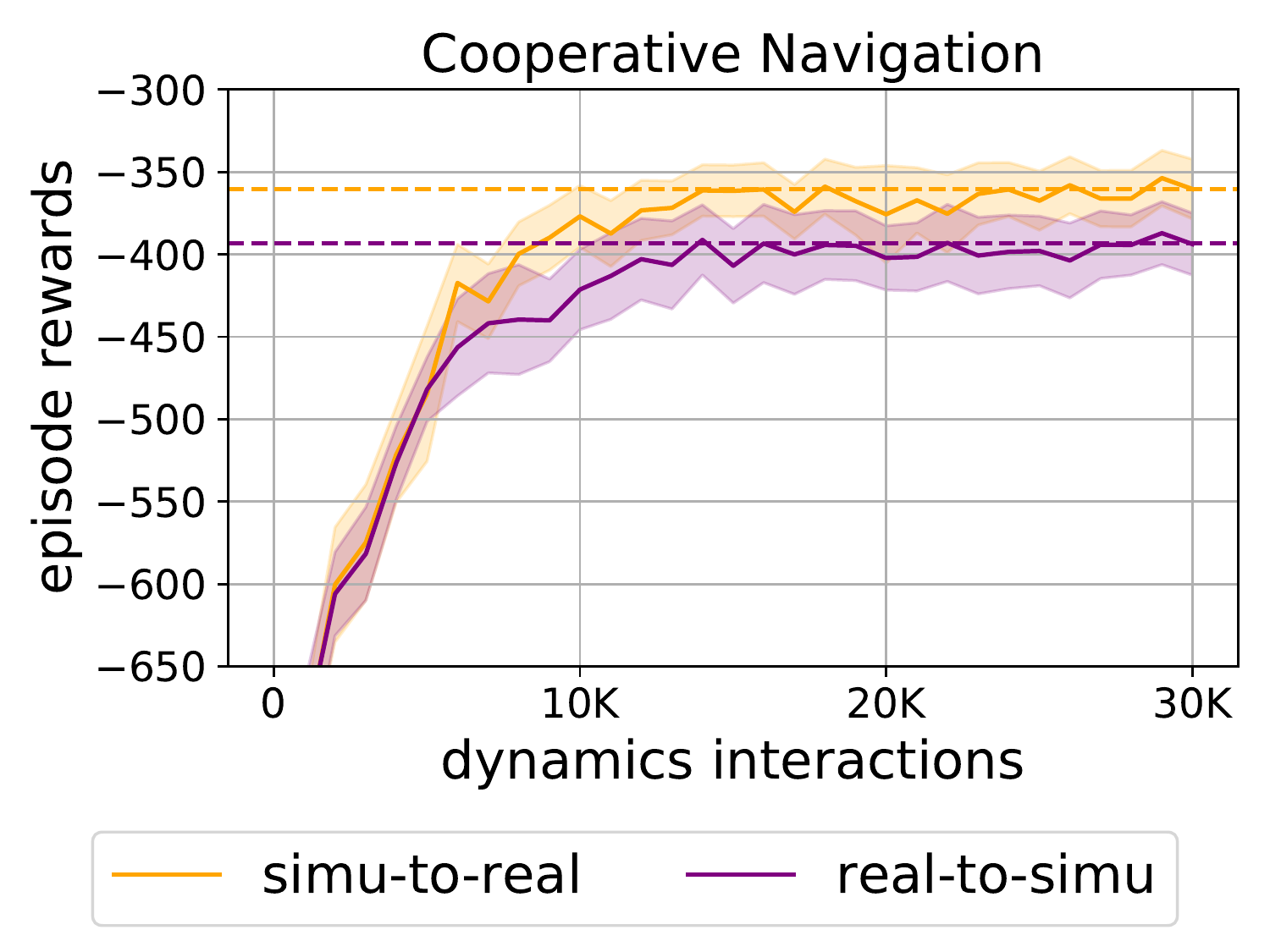}
	\captionof{figure}{The average episode reward performance  for two rollout schemes, using $k=10$.
	}
	\label{fig:two-rollout-performance}
	\vspace{10pt}
\end{minipage}

It may be counter-intuitive in the single-agent setting, while the situation is different in the multi-agent setting. 
We now provide the following detailed explanations for above observations. Although real-to-simu achieves less compounding error theoretically and empirically, yet the simulated rollouts generated by simu-to-real are more useful for policy training. If adopting the real-to-simu scheme, opponent models are exploited at the end when the error has been large already. In this situation, the dynamics model transits to useless states, and the opponent models provide unrealistic actions, which is of no use for ego policy training. Differently, if adopting the simu-to-real scheme, opponent models are exploited at the begging when the error is not large, and even if the error may be larger than in real-to-simu at the end, real opponents are used at this time, which at least guides the ego agent how to interact with others at these new states.


We think such an analysis requires a deeper investigation about the tigher compounding error bound w.r.t. the generalization error of the dynamics model and the opponent models, which deserves a thorough research work in the future.

